\documentclass{article}

\usepackage{microtype}
\usepackage{graphicx}
\usepackage{subfig}
\usepackage{booktabs} 
\usepackage{float}

\usepackage{amssymb,amsmath,amsthm}
\usepackage{mathtools}

\usepackage{bbm}
\usepackage{bm}
\usepackage{nicefrac}
\usepackage{color}
\usepackage{mdframed}
\usepackage{enumitem}
\usepackage{xspace}

\definecolor{teal}{rgb}{0.36, 0.54, 0.66}

\usepackage{hyperref}
\usepackage{cleveref}

\usepackage{autonum}

\usepackage[accepted]{icml2020}

\icmltitlerunning{Fast Differentiable Sorting and Ranking}

\DeclareMathOperator*{\argmax}{argmax}
\DeclareMathOperator*{\argmin}{argmin}
\DeclareMathOperator*{\dom}{dom}
\DeclareMathOperator*{\conv}{conv}

\newcommand\partialfrac[2]{\frac{\partial #1}{\partial #2}}

\def\cB{\mathcal{B}}

\def\cC{\mathcal{C}}
\def\cP{\mathcal{P}}
\def\cS{\mathcal{S}}
\def\cV{\mathcal{V}}

\def\cX{\mathcal{X}}
\def\cY{\mathcal{Y}}

\def\RR{\mathbb{R}}

\def\bnu{\bm{\nu}}
\def\btheta{\bm{\theta}}
\def\bsigma{\bm{\sigma}}

\def\bmu{\bm{\mu}}

\def\blambda{\bm{\lambda}}

\def\B{\bm{B}}
\def\I{\bm{I}}
\def\J{\bm{J}}
\def\P{\bm{P}}

\def\a{\bm{a}}
\def\b{\bm{b}}
\def\f{\bm{f}}

\def\r{\bm{r}}
\def\s{\bm{s}}

\def\u{\bm{u}}
\def\v{\bm{v}}
\def\w{{\bm{w}}}
\def\x{\bm{x}}
\def\y{{\bm{y}}}
\def\z{\bm{z}}
\def\ones{\mathbf{1}}
\def\zeros{\mathbf{0}}
\def\epsmin{\varepsilon_{\text{min}}}
\def\epsmax{\varepsilon_{\text{max}}}

\def\re{\rho}
\def\bre{\bm{\rho}}

\def\sigmainv{{\sigma^{-1}}}

\def\argsort{\textbf{argsort}\xspace}
\def\ranks{\textbf{rank}\xspace}

\definecolor{theoremcolor}{rgb}{1.0, 1.0, 1.0}
\mdfsetup{
    backgroundcolor=theoremcolor,
    linewidth=0.5pt,
}

\newmdtheoremenv{proposition}{Proposition}
\newmdtheoremenv{lemma}{Lemma}
\newmdtheoremenv{corollary}{Corollary}

\begin{document}

\twocolumn[
\icmltitle{Fast Differentiable Sorting and Ranking}



\icmlsetsymbol{equal}{*}

\begin{icmlauthorlist}
\icmlauthor{Mathieu Blondel}{goo}
\icmlauthor{Olivier Teboul}{goo}
\icmlauthor{Quentin Berthet}{goo}
\icmlauthor{Josip Djolonga}{goo}
\end{icmlauthorlist}

\icmlaffiliation{goo}{Google Research, Brain team}

\icmlcorrespondingauthor{Mathieu Blondel}{mblondel@google.com}
\icmlcorrespondingauthor{Olivier Teboul}{oliviert@google.com}
\icmlcorrespondingauthor{Quentin Berthet}{qberthet@google.com}
\icmlcorrespondingauthor{Josip Djolonga}{josipd@google.com}

\icmlkeywords{sorting, ranking, permutahedron, isotonic regression}

\vskip 0.3in
]



\printAffiliationsAndNotice{}  

\begin{abstract}
The sorting operation is one of the most commonly used building blocks
in computer programming. In machine learning, it is often
used for robust statistics. However, seen
as a function, it is piecewise linear and as a result includes many kinks 
where it is non-differentiable. More problematic is the related ranking operator,
often used for order statistics and ranking metrics.
It is a piecewise constant function, meaning that its derivatives are null or
undefined.  While numerous works have proposed differentiable proxies to sorting
and ranking, they do not achieve the $O(n \log n)$ time complexity one would
expect from sorting and ranking operations. In this paper, we propose the first
differentiable sorting and ranking operators with $O(n \log n)$ time and $O(n)$
space complexity.
Our proposal in addition enjoys exact computation and differentiation.
We achieve this feat by constructing differentiable
operators as projections onto the permutahedron, the convex hull of permutations,
and using a reduction to isotonic optimization.
Empirically, we confirm that our approach is an order of magnitude faster than
existing approaches and showcase two novel applications: differentiable
Spearman's rank correlation coefficient and least trimmed squares.
\end{abstract}

\vspace{-0.5cm}
\section{Introduction}

Modern deep learning architectures are built by composing parameterized
functional blocks (including loops and conditionals) and are trained end-to-end
using gradient backpropagation. This has motivated the term
\textbf{differentiable programming}, recently popularized, among others, by
\citet{lecun_2018}. Despite great empirical successes, many operations 
commonly used in computer programming remain poorly differentiable or
downright pathological, limiting the set of architectures for which a gradient
can be computed.

We focus in this paper on two such operations: \textbf{sorting} and
\textbf{ranking}. Sorting
returns the given input vector with its values re-arranged in monotonic order. 
It plays a key role to handle outliers in robust statistics, as in least-quantile
\citep{rousseeuw_1984} or trimmed \citep{rousseeuw_2005} regression. As a
piecewise linear function, however, the sorted vector contains many kinks where
it is non-differentiable. In addition, when used in composition with other
functions, sorting often induces non-convexity, thus rendering model
parameter optimization difficult.

The ranking operation, on the other hand, outputs the positions, or ranks, of
the input values in the sorted vector. A workhorse of
order statistics \citep{david_2004}, ranks are used in several
metrics, including Spearman's rank correlation coefficient
\citep{spearman_1904}, top-$k$ accuracy and normalized discounted cumulative
gain (NDCG). As piecewise constant functions, ranks are
unfortunately much more problematic than sorting: their derivatives are null or undefined,
preventing gradient backpropagation. For this reason, a large body of work has
studied differentiable proxies to ranking. While several works opt 
to approximate ranking metrics directly
\citep{chapelle_2010,sinkprop,lapin_2016, rolinek_2020},
others introduce ``soft'' ranks, which can then be plugged into
any differentiable loss function.  \citet{taylor_2008} use a random perturbation
technique to compute expected ranks in $O(n^3)$ time, where $n$ is the
dimensionality of the vector to rank.
\citet{qin_2010} propose a simple method based on comparing
pairwise distances between values, thereby taking
$O(n^2)$ time. This method is refined by \citet{grover_2019} using unimodal
row-stochastic matrices. Lastly, \citet{cuturi_2019} adopt an optimal transport
viewpoint of sorting and ranking. Their method is based on differentiating
through the iterates of the Sinkhorn algorithm \citep{sinkhorn} and costs
$O(Tmn)$ time, where $T$ is the number of Sinkhorn
iterations and $m \in \mathbb{N}$ is a hyper-parameter which trades computational
cost and precision (convergence to ``hard'' sort and ranks is only guaranteed if
$m=n$).  

In this paper, we propose the first differentiable sorting and ranking operators
with $O(n \log n)$ time and $O(n)$ memory complexity. Our proposals enjoy
\textbf{exact} computation and differentiation (i.e., they do not involve
differentiating through the iterates of an approximate algorithm).
We achieve this feat by casting differentiable sorting and ranking as
projections onto the permutahedron, the convex hull of all permutations,
and using a reduction to isotonic optimization.
While the permutahedron had been used for learning before
\citep{yasutake_2011,ailon_2016,blondel_2019}, it had not been used 
to define fast differentiable operators.
The rest of the paper is organized as follows.
\begin{itemize}[topsep=0pt,itemsep=2pt,parsep=2pt,leftmargin=10pt]

    \item We review the necessary background (\S\ref{sec:preliminaries}) and 
	show how to cast sorting and ranking as linear programs over the
        permutahedron, the convex hull of all permutations (\S\ref{sec:lp_formulations}).

    \item We introduce regularization in these linear programs, which turns them
        into projections onto the permutahedron and allows us to define
        differentiable sorting and ranking operators.
        We analyze the properties of these operators, such as their asymptotic
        behavior (\S\ref{sec:soft_ops}).

    \item Using a reduction to isotonic optimization, we achieve $O(n \log n)$
        computation and $O(n)$ differentiation of our operators, a key technical
    contribution of this paper (\S\ref{sec:fast_comp_diff}).

    \item We show that our approach is an order of magnitude faster than
existing approaches and showcase two novel applications: differentiable
Spearman's rank coefficient and soft least trimmed squares
(\S\ref{sec:experiments}).

\end{itemize}

\section{Preliminaries}
\label{sec:preliminaries}

In this section, we define the notation that will be used throughout this paper.
Let $\btheta \coloneqq (\theta_1, \dots, \theta_n) \in \RR^n$. We
will think of $\btheta$ as a vector of scores or ``logits'' produced by a
model, i.e., $\btheta \coloneqq g(\x)$ for some $g \colon \cX \to \RR^n$ and
some $\x \in \cX$.  For instance, in a label ranking setting, $\btheta$ may
contain the score of each of $n$ labels for the features $\x$.

We denote a \textbf{permutation} of $[n]$ by $\sigma = (\sigma_1, \dots,
\sigma_n)$ and its inverse by $\sigma^{-1}$.  For convenience, we will sometimes
use $\pi \coloneqq \sigmainv$.  If a permutation $\sigma$ is seen as a
vector, we denote it with bold, $\bsigma \in [n]^n$.  We denote the set of $n!$
permutations of $[n]$ by $\Sigma$.
Given a permutation $\sigma \in \Sigma$, we denote the version of
$\btheta = (\theta_1, \dots, \theta_n) \in \RR^n$ permuted according to $\sigma$
by $\btheta_\sigma \coloneqq (\theta_{\sigma_1}, \dots, \theta_{\sigma_n}) \in
\RR^n$.
We define the reversing permutation by 
$\re \coloneqq (n, n-1, \dots, 1)$ or $\bre$ in vector form.
Given a set $\cS \subseteq [n]$ and a vector $\v \in \RR^n$, we denote 
the restriction of $\v$ to $\cS$ by 
$\v_{\cS} \coloneqq (v_i \colon i \in \cS) \in \RR^{|\cS|}$.

We define the \argsort of $\btheta$ as the indices sorting $\btheta$,
i.e.,
\begin{equation}
\sigma(\btheta) \coloneqq (\sigma_1(\btheta), \dots,
\sigma_n(\btheta)),
\end{equation}
where $\theta_{\sigma_1(\btheta)} \ge \dots \ge \theta_{\sigma_n(\btheta)}$.
If some of the coordinates of $\btheta$ are equal, we break ties arbitrarily. 
We define the \textbf{sort} of $\btheta$ as the values of $\btheta$ in
descending order, i.e.,
\begin{equation}
    s(\btheta) \coloneqq \btheta_{\sigma(\btheta)}.
\end{equation}
We define the \ranks of $\btheta$ as the function evaluating at coordinate $j$ to the position of $\theta_j$ in the descending sort (smaller rank $r_j(\btheta)$ means that $\theta_j$ has higher
value). It is formally equal to the argsort's inverse permutation, i.e.,
\begin{equation}
    r(\btheta) \coloneqq \sigma^{-1}(\btheta).
\end{equation}
For instance, if $\theta_3 \geq \theta_1 \geq \theta_2$, then 
$\sigma(\btheta) = (3, 1, 2)$, 
$s(\btheta) = (\theta_3, \theta_1, \theta_2)$ and
$r(\btheta) = (2, 3, 1)$.
All three operations can be computed in $O(n \log n)$ time.
Note that throughout this paper, we use descending order for convenience.
The ascending order counterparts are easily obtained by 
$\sigma(-\btheta)$, $-s(-\btheta)$ and $r(-\btheta)$, respectively.

\section{Sorting and ranking as linear programs}
\label{sec:lp_formulations}

We show in this section how to cast sorting and ranking operations
as linear programs over the permutahedron.
To that end, we first formulate the argsort and ranking operations as
optimization problems over the set of permutations $\Sigma$.
\begin{lemma}{Discrete optimization formulations}
\label{lemma:discrete}

For all $\btheta \in \RR^n$ and
$\bre \coloneqq (n, n-1, \dots, 1)$, we have
\begin{align}
\sigma(\btheta) 
&= \argmax_{\sigma \in \Sigma} \langle \btheta_\sigma, \bre \rangle
\label{eq:sigma_discrete}, \textrm{ and} \\
r(\btheta) 
&= \argmax_{\pi \in \Sigma} \langle \btheta, \bre_\pi \rangle.
\label{eq:r_discrete}
\end{align}
\end{lemma}
A proof is provided in \S\ref{appendix:proof_lemma_discrete}.
To obtain continuous optimization problems,
we introduce the permutahedron induced by a vector $\w \in \RR^n$, the
convex hull of permutations of $\w$:
\begin{equation}
\cP(\w) \coloneqq \conv(\{ \w_\sigma \colon \sigma \in \Sigma \}) \subset \RR^n.
\end{equation}
A well-known object in combinatorics \citep{bowman_1972,ziegler_2012}, the
permutahedron of $\w$ is a convex polytope, whose vertices correspond to
permutations of $\w$.  It is illustrated in Figure \ref{fig:permutahedron}.  In
particular, when $\w = \bre$, $\cP(\w) = \conv(\Sigma)$.  With this defined, we
can now derive linear programming formulations of sort and ranks.
\begin{proposition}{Linear programming formulations}
\label{prop:lp}

For all $\btheta \in \RR^n$ and
$\bre \coloneqq (n, n-1, \dots, 1)$, we have
\begin{align}
s(\btheta) 
&= \argmax_{\y \in \cP(\btheta)} \langle \y, \bre \rangle 
\label{eq:s_lp}, \textrm{ and} \\
r(\btheta) 
&= \argmax_{\y \in \cP(\bre)} \langle \y, -\btheta \rangle.
\label{eq:r_lp}
\end{align}
\end{proposition}
A proof is provided in \S\ref{appendix:proof_prop_lp}.  The key idea is to
perform a change of variable to ``absorb'' the permutation in
\eqref{eq:sigma_discrete} and \eqref{eq:r_discrete} into a
permutahedron. From the fundamental theorem of linear programming \citep[Theorem
6]{dantzig}, an optimal solution of a linear program is almost surely achieved
at a vertex of the convex polytope, a permutation in the case of the
permutahedron.  Interestingly, $\btheta$ appears in the
constraints and $\bre$ appears in the objective for sorting, while this is the
opposite for ranking.
\begin{figure}[t]
    \centering
    \includegraphics[height=4.5cm]{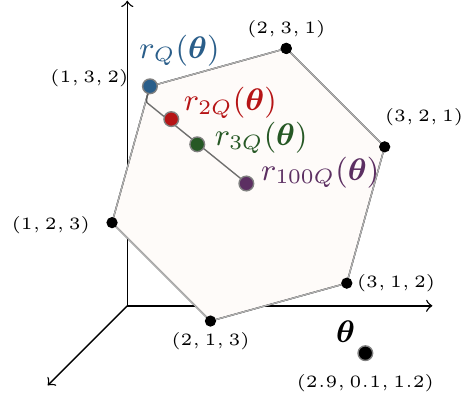} 
    \caption{{\bf Illustration of the permutahedron} $\cP(\bre)$, whose vertices
        are permutations of $\bre = (3, 2, 1)$. In this example, the
        ranks of $\btheta = (2.9, 0.1, 1.2)$ are $r(\btheta) = (1, 3, 2)$.
        In this case, our proposed soft rank $r_{\varepsilon Q}(\btheta)$ with
        $\varepsilon = 1$ is exactly equal to $r(\btheta)$.  
        When $\varepsilon \to \infty$, $r_{\varepsilon Q}(\btheta)$ converges
        towards the centroid of the permutahedron. The gray line indicates the
        regularization path of $r_{\varepsilon Q}(\btheta)$ between these two
        regimes, when varying $\varepsilon$.
}
\label{fig:permutahedron}
\end{figure}

\paragraph{Differentiability a.e.\ of sorting.}

For $s(\btheta)$, the fact that $\btheta$ appears in the linear program
constraints
makes $s(\btheta)$ piecewise linear and thus differentiable almost everywhere. 
When $\sigma(\btheta)$ is unique at $\btheta$, $s(\btheta) =
\btheta_{\sigma(\btheta)}$ is differentiable at $\btheta$ and its Jacobian
is the permutation matrix associated with $\sigma(\btheta)$.
When $\sigma(\btheta)$ is not unique, we can choose any matrix in
Clarke’s generalized Jacobian, i.e., any convex combination
of the permutation matrices associated with $\sigma(\btheta)$.

\paragraph{Lack of useful Jacobian of ranking.}

On the other hand, for $r(\btheta)$, since $\btheta$ appears in the objective, a
small perturbation to $\btheta$ may cause the solution of the linear program to
jump to another permutation of $\bre$.  This makes $r(\btheta)$ a discontinuous,
piecewise constant function. This means that $r(\btheta)$ has null or undefined
partial derivatives, preventing its use within a neural network trained with
backpropagation.

\section{Differentiable sorting and ranking}
\label{sec:soft_ops}

\begin{figure}[t]
    \centering
    \includegraphics[height=4.5cm]{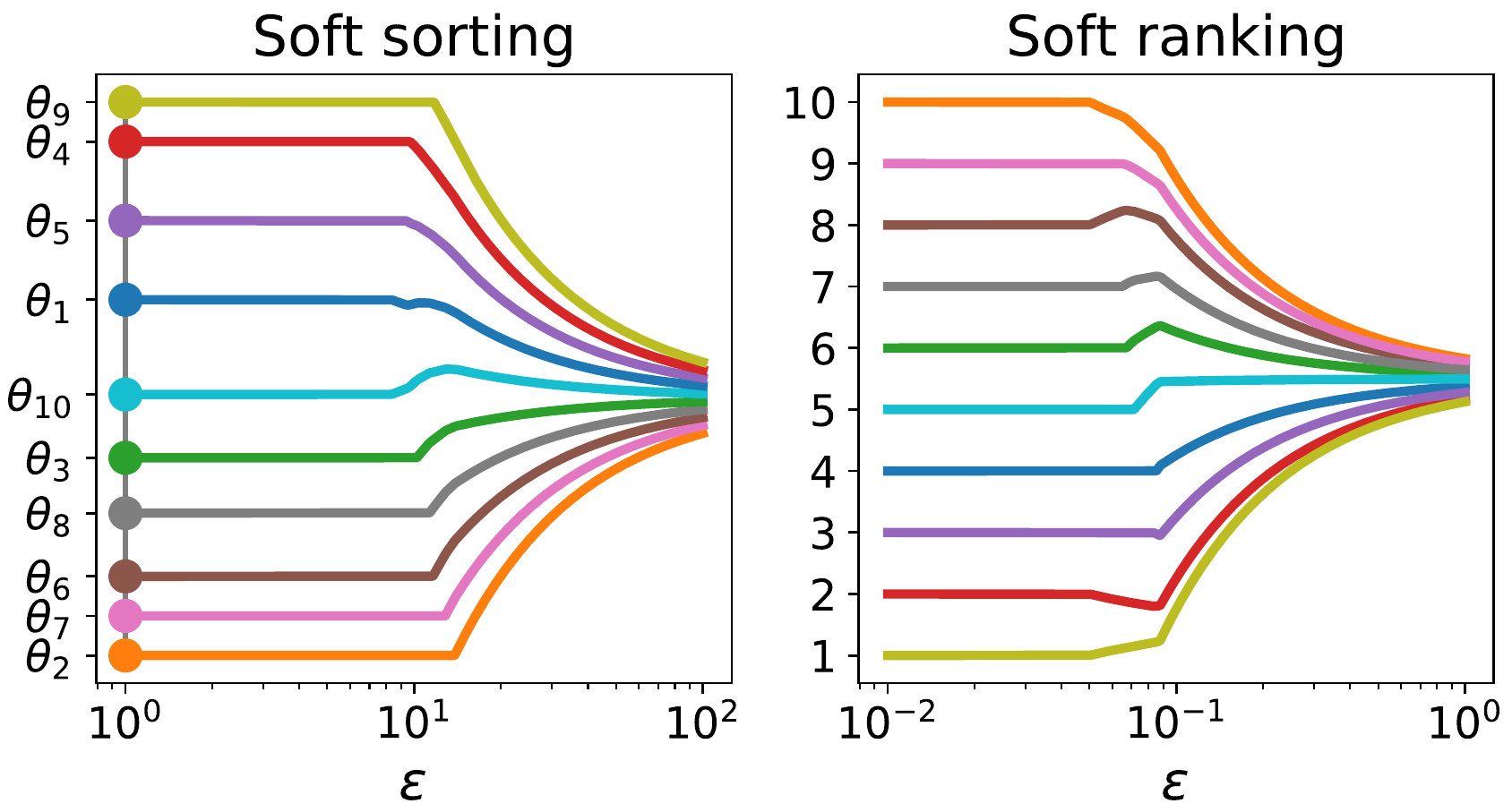} 
    \caption{Illustration of the soft sorting and ranking operators,
	    $s_{\varepsilon \Psi}(\btheta)$ and $r_{\varepsilon \Psi}(\btheta)$
	    for $\Psi=Q$; the results with $\Psi=E$ are similar. 
	    When $\varepsilon \to 0$, they converge to their ``hard''
	    counterpart. When $\varepsilon \to \infty$, they collapse into a
	    constant, as proven in Prop.\ref{prop:properties}.}
	    \label{fig:soft-illustration}
\end{figure}

As we have already motivated, our primary goal is the design of efficiently
computable approximations to the sorting and ranking operators, that would
smoothen the numerous kinks of the former, and provide useful derivatives for
the latter.  We achieve this by introducing strongly convex regularization in our linear
programming formulations. This turns them into efficiently computable projection
operators, which are differentiable and amenable to formal analysis.  
\paragraph{Projection onto the permutahedron.} 

Let $\z, \w \in \RR^n$ and consider the linear program  
$\argmax_{\bmu \in \cP(\w)} \langle \bmu, \z \rangle$. 
Clearly, we can express $s(\btheta)$ by setting $(\z,\w) = (\bre,\btheta)$ and
$r(\btheta)$ by setting $(\z,\w) = (-\btheta,\bre)$.  Introducing quadratic
regularization $Q(\bmu) \coloneqq \frac{1}{2} \|\bmu\|^2$ 
is considered by \citet{sparsemax} over the unit simplex
and by \citet{sparsemap} over marginal polytopes.
Similarly, adding $Q$ to our linear program over the permutahedron gives
\begin{equation}
P_Q(\z, \w) \coloneqq 
\argmax_{\bmu \in \cP(\w)} \langle \z, \bmu \rangle - Q(\bmu)
= \argmin_{\bmu \in \cP(\w)} \frac{1}{2} \|\bmu - \z\|^2,
\end{equation}
i.e., the Euclidean projection of $\z$ onto $\cP(\w)$.
We also consider entropic regularization
$E(\bmu) \coloneqq \langle \bmu, \log \bmu - \ones \rangle$, popularized in the
optimal transport literature \citep{cuturi2013sinkhorn,peyre_2017}. 
Subtly, we define
\begin{align}
P_E(\z, \w) 
&\coloneqq \log \argmax_{\bmu \in \cP(e^\w)} \langle \z, \bmu \rangle -
E(\bmu) \\
&= \log \argmin_{\bmu \in \cP(e^\w)}
\textnormal{KL}(\bmu, e^{\z}),
\end{align}
where 
$\text{KL}(\bm{a}, \bm{b}) \coloneqq \sum_i a_i \log \frac{a_i}{b_i} - \sum_i
a_i + \sum_i b_i$ is the Kullback-Leibler (KL) divergence between two positive
measures $\bm{a} \in \RR^n_+$ and $\bm{b} \in \RR^n_+$.  $P_E(\z, \w)$ is
therefore the \textit{log} KL projection of $e^{\z}$ onto $\cP(e^\w)$.
The purpose of $e^\w$ is to ensure that $\bmu$ always belongs
to $\dom(E) = \RR_+^n$ (since $\bmu$ is a convex combination of the permutations
of $e^\w$) and that of the logarithm is to map $\bmu^\star$ back to $\RR^n$. 

More generally, we can use any strongly convex regularization $\Psi$ under mild
conditions. For concreteness, we focus our exposition in the main text on $\Psi
\in \{Q,E\}$. We state all our propositions for these two cases
and postpone a more general treatment to the appendix.

\paragraph{Soft operators.} 

We now build upon these projections to define soft sorting and ranking
operators.  To control the regularization strength, we introduce a parameter
$\varepsilon > 0$ which we multiply $\Psi$ by (equivalently, divide $\z$ by).

For sorting, we choose
$(\z,\w)=(\bre,\btheta)$ and therefore
define the $\Psi$-regularized soft sort as 
\begin{equation}
s_{\varepsilon \Psi}(\btheta) 
\coloneqq P_{\varepsilon \Psi}(\bre, \btheta)
= P_{\Psi}(\nicefrac{\bre}{\varepsilon}, \btheta).
\label{eq:soft_sort}
\end{equation}

\begin{figure*}[t]
    \centering
    \includegraphics[width=0.95\textwidth]{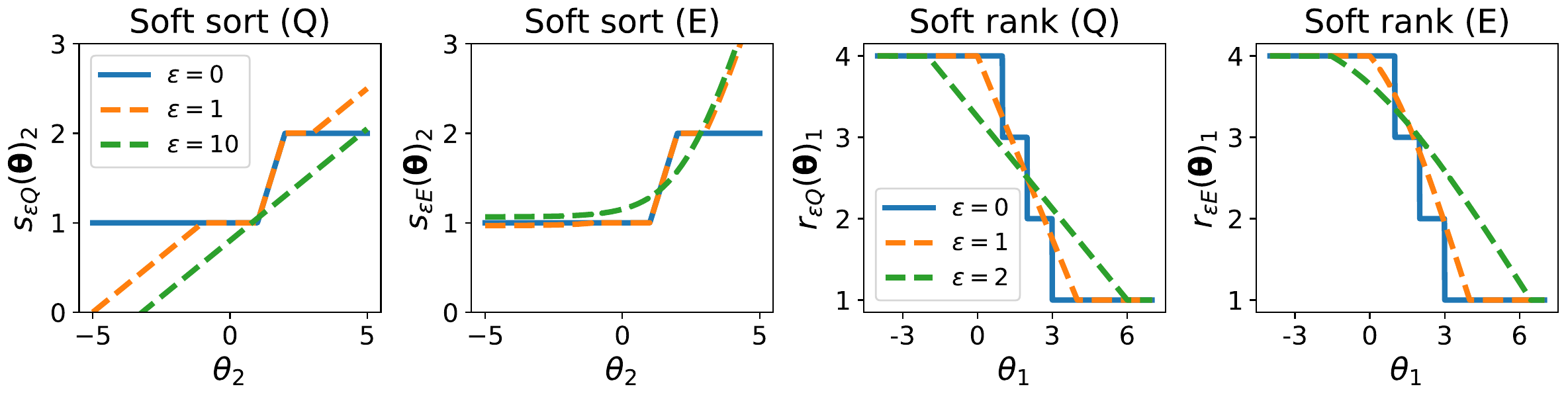} 
    \caption{{\bf Effect of the regularization parameter $\varepsilon$.}
We take the vector $\btheta \coloneqq (0, 3, 1, 2)$, vary one of its coordinates
$\theta_i$ and look at how $[s_{\varepsilon \Psi}(\btheta)]_i$ and
$[r_{\varepsilon \Psi}(\btheta)]_i$ change in response.
For soft sorting with $\Psi=Q$, the function is still piecewise linear, like
sorting. However, by increasing $\varepsilon$ we reduce the number of kinks,
and the function eventually converges to a mean (Proposition
\ref{prop:properties}). With $\Psi=E$, the function tends to be even smoother.
For soft ranking with $\Psi=Q$, the function is piecewise
linear instead of piecewise constant for the ``hard'' ranks. With $\Psi=E$, the function again tends to
be smoother though it may contain kinks.
}
    \label{fig:smoothness}
\end{figure*}

For ranking, we choose $(\z,\w)=(-\btheta, \bre)$ and therefore define the
$\Psi$-regularized soft rank as
\begin{equation}
r_{\varepsilon \Psi}(\btheta) 
\coloneqq P_{\varepsilon \Psi}(-\btheta, \bre)
= P_\Psi(\nicefrac{-\btheta}{\varepsilon}, \bre).
\label{eq:soft_rank}
\end{equation}
We illustrate the behavior of both of these soft operations as we vary
$\varepsilon$ in \Cref{fig:soft-illustration,fig:smoothness}.
As for the hard versions, the ascending-order soft sorting and
ranking are obtained by negating the input as
$-s_{\varepsilon \Psi}(-\btheta)$ and $r_{\varepsilon \Psi}(-\btheta)$,
respectively.

\paragraph{Properties.}

We can further characterize these approximations.
Namely, as we now formalize, they are differentiable a.e., and not only
converge to the their ``hard'' counterparts, but also satisfy some of their
properties for all $\varepsilon$.
\begin{proposition}{Properties of $s_{\varepsilon \Psi}(\btheta)$ and
	$r_{\varepsilon \Psi}(\btheta)$}\label{prop:properties}

\vspace{-0.5cm}
~ 
\begin{enumerate}[topsep=0pt,itemsep=3pt,parsep=3pt,leftmargin=15pt]
\item {\bf Differentiability.}
For all $\varepsilon > 0$, $s_{\varepsilon \Psi}(\btheta)$ and $r_{\varepsilon
\Psi}(\btheta)$ are differentiable (a.e.) w.r.t.\ $\btheta$.

\item {\bf Order preservation.} 
Let $\s \coloneqq s_{\varepsilon \Psi}(\btheta)$,
$\r \coloneqq r_{\varepsilon \Psi}(\btheta)$ and
$\sigma \coloneqq \sigma(\btheta)$.
For all $\btheta \in \RR^n$ and $0 < \varepsilon < \infty$, we have
$s_1 \ge s_2 \ge \dots \ge s_n$
and 
$r_{\sigma_1} \le r_{\sigma_2} \le \dots \le r_{\sigma_n}$.

\item {\bf Asymptotics.} 
For all $\btheta \in \RR^n$ without ties:
\begin{equation}
\begin{aligned}
s_{\varepsilon \Psi}(\btheta) 
&\xrightarrow[\varepsilon \to 0]{} s(\btheta)
&r_{\varepsilon \Psi}(\btheta) 
&\xrightarrow[\varepsilon \to 0]{} r(\btheta) \\
&\xrightarrow[\varepsilon \to \infty]{} f_\Psi(\btheta) \ones
&&\xrightarrow[\varepsilon \to \infty]{} f_\Psi(\bre) \ones,
\end{aligned}
\end{equation}
where
$f_Q(\u) \coloneqq \text{mean}(\u)$,
$f_E(\u) \coloneqq \log f_Q(\u)$.
\end{enumerate}
\end{proposition}
The last property describes the behavior as $\varepsilon \to 0$ and $\varepsilon
\to \infty$.  Together with the proof of \Cref{prop:properties}, we include in
\S\ref{appendix:prop_soft_sort_rank} a slightly stronger result. Namely, we
derive an explicit value of $\varepsilon$ below which our operators are exactly
equal to their hard counterpart, and a value of $\varepsilon$ above which our
operators can be computed in closed form.

\paragraph{Convexification effect.}

Proposition \ref{prop:properties} shows that $[s_{\varepsilon \Psi}(\btheta)]_i$
and $[r_{\varepsilon \Psi}(\btheta)]_i$ for all $i \in [n]$ converge to convex
functions of $\btheta$ as $\varepsilon \to \infty$. This suggests that larger
$\varepsilon$ make the objective function increasingly easy to optimize
(at the cost of departing from ``hard'' sorting or ranking).
This behavior is also visible in \Cref{fig:smoothness}, where
$[s_{\varepsilon Q}(\btheta)]_2$ converges towards the mean $f_Q$, depicted by a
straight line.

\paragraph{On tuning $\varepsilon$ (or not).}

The parameter $\varepsilon > 0$ controls the trade-off between approximation
of the original operator and ``smoothness''. When the model $g(\x)$ producing
the scores or ``logits'' $\btheta$ to be sorted/ranked is a homogeneous
function, from \eqref{eq:soft_sort} and \eqref{eq:soft_rank},
$\varepsilon$ can be absorbed into the model. In our label ranking experiment,
we find that indeed tuning $\varepsilon$ is not necessary to achieve excellent
accuracy. On the other hand, for top-$k$ classification, we find that applying a
logistic map to squash $\btheta$ to $[0,1]^n$ and tuning $\varepsilon$ is
important, confirming the empirical finding of \citet{cuturi_2019}.

\paragraph{Relation to linear assignment formulation.}

We now discuss the relation between our proposal and a formulation
based on the Birkhoff polytope $\cB \subset \RR^{n \times n}$, the convex
hull of permutation matrices. 
Our exposition corresponds to the method of \citet{cuturi_2019} with $m=n$.
Note that using the change of variable $\y = \P \bre$ and $\cP(\bre) = \cB
\bre$, we can rewrite \eqref{eq:r_lp} as $r(\btheta) = \P(\btheta) \bre$, where
\begin{equation}
\P(\btheta) 
\coloneqq \argmax_{\P \in \cB} ~ \langle \P \bre, -\btheta \rangle.
\end{equation}
Let $D(\a, \b) \in \RR^{n \times n}$ be a distance matrix.
Simple calculations show that if
$[D(\a, \b)]_{i,j} \coloneqq \frac{1}{2} (a_i - b_j)^2$, then
\begin{equation}
\P(\btheta) 
= \argmin_{\P \in \cB} ~ \langle \P, D(-\btheta, \bre) \rangle.
\end{equation}
Similarly, we can rewrite \eqref{eq:s_lp} as $s(\btheta) = \P(\btheta)^\top
\btheta$. To obtain a differentiable operator, \citet{cuturi_2019} (see also
\citep{sinkprop}) propose to replace the permutation matrix $\P(\btheta)$ by a
doubly stochastic matrix $\P_{\varepsilon
E}(\btheta) \coloneqq \argmin_{\P \in \cB} ~ \langle \P, D(-\btheta,
\bre) \rangle + \varepsilon E(\P)$, which is computed approximately in $O(T
n^2)$ using Sinkhorn (\citeyear{sinkhorn}). In comparison, our approach is based
on regularizing $\y = \P \bre$ with $\Psi \in \{Q,E\}$ directly, the key to
achieve $O(n \log n)$ time and $O(n)$ space complexity, as we now show.

\section{Fast computation and differentiation}
\label{sec:fast_comp_diff}

As shown in the previous section, computing our soft sorting and ranking
operators boils down to projecting onto a permutahedron.
Our key contribution in this section is the derivation of an $O(n \log n)$
forward pass and an $O(n)$ backward pass (multiplication with the Jacobian)
for these projections. Beyond soft sorting and ranking,
this is an important sensitivity analysis question in its own right.

\paragraph{Reduction to isotonic optimization.}

We now show how to reduce the projections to isotonic optimization, i.e.,
with simple chain constraints, which is the key to fast computation and
differentiation. We will w.l.o.g.\ assume that $\w$ is sorted in descending
order (if not the case, we sort it first).
\begin{proposition}{Reduction to isotonic optimization}
\label{prop:projection}

For all $\z \in \RR^n$ and sorted $\w \in \RR^n$ we have
\begin{equation}
    P_\Psi(\z, \w) = \z - \v_\Psi(\z_{\sigma(\z)}, \w)_{\sigmainv(\z)}
\end{equation}
where
\begin{align}
\v_Q(\s, \w) &\coloneqq
\argmin_{v_1 \ge \dots \ge v_n} \frac{1}{2} \|\v - (\s - \w)\|^2, \textrm{ and}\\
\v_E(\s, \w) &\coloneqq
\argmin_{v_1 \ge \dots \ge v_n} 
\langle e^{\s - \v}, \ones \rangle +  \langle e^\w,
\v \rangle.
\end{align}
\end{proposition}
The function $\v_Q$ is classically known as isotonic regression.
The fact that it can be used to solve the Euclidean projection onto $\cP(\w)$ 
has been noted several times \citep{orbit_regul,zeng_2014}.
The reduction of Bregman projections, which we use here, to isotonic
optimization was shown by \citet{projection_permutahedron}. 
Unlike that study, we use the KL projection of $e^{\z}$ onto
$\cP(e^\w)$, and not of $\z$ onto $\cP(\w)$, which simplifies many expressions.
We include in \S\ref{appendix:proof_forward} a
simple unified proof of \Cref{prop:projection}
based on Fenchel duality and tools from submodular optimization.
We also discuss an interpretation of adding regularization to the primal
linear program as relaxing the equality constraints of the dual linear program
in \S\ref{appendix:relaxed_dual_lp}.

\paragraph{Computation.}

As shown by \citet{best_2000_separable},
the classical pool adjacent violators (PAV) algorithm
for isotonic regression can be extended to minimize any
per-coordinate decomposable convex function $f(\v)=\sum_{i=1}^n f_i(v_i)$
subject to monotonicity constraints, which is exactly the form of the problems
in \Cref{prop:projection}.
The algorithm repeatedly splits the coordinates into
a set of contiguous blocks $\cB_1,\dots,\cB_m$ that partition $[n]$ (their
union is $[n]$  and $\max\cB_j+1=\min\cB_{j+1}$).
It only requires access to an oracle that solves for each block $\cB_j$ the
sub-problem $\gamma(\cB_j)=\argmin_{\gamma \in \RR} \sum_{i\in\cB_j} f_i(\gamma)$, 
and runs in \textbf{linear} time.
Further, the solution has a clean block-wise constant structure, namely it is
equal to $\gamma(\cB_j)$ within block $\cB_j$.
Fortunately, in our case, as shown in \S\ref{appendix:pav}, the function
$\gamma$ can be analytically computed, as
\begin{align}
\gamma_Q(\cB_j; \s, \w) &\coloneqq
\frac{1}{|\cB_j|} \sum_{i \in \cB_j} s_i - w_i 
\label{eq:gamma_quadratic}, \textrm{ and} \\
\gamma_E(\cB_j; \s, \w) &\coloneqq 
\log \sum_{i \in \cB_j} e^{s_i} - \log \sum_{i \in \cB_j} e^{w_i}.
\label{eq:gamma_entropic}
\end{align}
Hence, PAV returns an \textbf{exact} solution of both $\v_Q(\s, \w)$ and
$\v_E(\s, \w)$ in $O(n)$ time \citep{best_2000_separable}.
This means that we do not need to choose a number of iterations or a
level of precision, unlike with Sinkhorn.
Since computing $P_Q(\z, \w)$ and
$P_E(\z, \w)$ requires obtaining $\s = \z_{\sigma(\btheta)}$ beforehand, the
total computational complexity is $O(n \log n)$. 

\paragraph{Differentiating isotonic optimization.}
The block-wise structure of the solution also makes its derivatives easy to
analyze, despite the fact that we are differentiating the \emph{solution} of an
optimization problem.
Since the coordinates of the solution in block $\cB_j$ are all equal to
$\gamma(\cB_j)$, which in turn depends only on a subset of the
parameters, the Jacobian has a simple block-wise form, which we now formalize.
\begin{lemma}{Jacobian of isotonic optimization}
\label{lemma:Jacobian_isotonic}

Let $\cB_1,\dots,\cB_m$ be the ordered partition of $[n]$ induced by
$\v_\Psi(\s, \w)$ from Proposition \ref{prop:projection}. Then,
\begin{equation}
    \frac{\partial \v_\Psi(\s, \w)}{\partial \s} = 
\begin{bmatrix} 
    \mathbf{B}^\Psi_1 & \zeros & \zeros \\
    \zeros            & \ddots & \zeros \\
    \zeros            & \zeros & \mathbf{B}^\Psi_m
\end{bmatrix}
\in \RR^{n \times n},
\end{equation}
where
$\mathbf{B}_j^\Psi 
\coloneqq
\partial \gamma_\Psi(\cB_j; \s, \w) / \partial \s
\in \RR^{|\cB_j| \times |\cB_j|}$.
\end{lemma}
A proof is given in \S\ref{appendix:Jacobian_isotonic}.
The Jacobians w.r.t. $\w$ are
entirely similar, thanks to the symmetry of \eqref{eq:gamma_quadratic}
and \eqref{eq:gamma_entropic}.

\begin{figure*}[th]
    \centering
    \includegraphics[height=3.8cm]{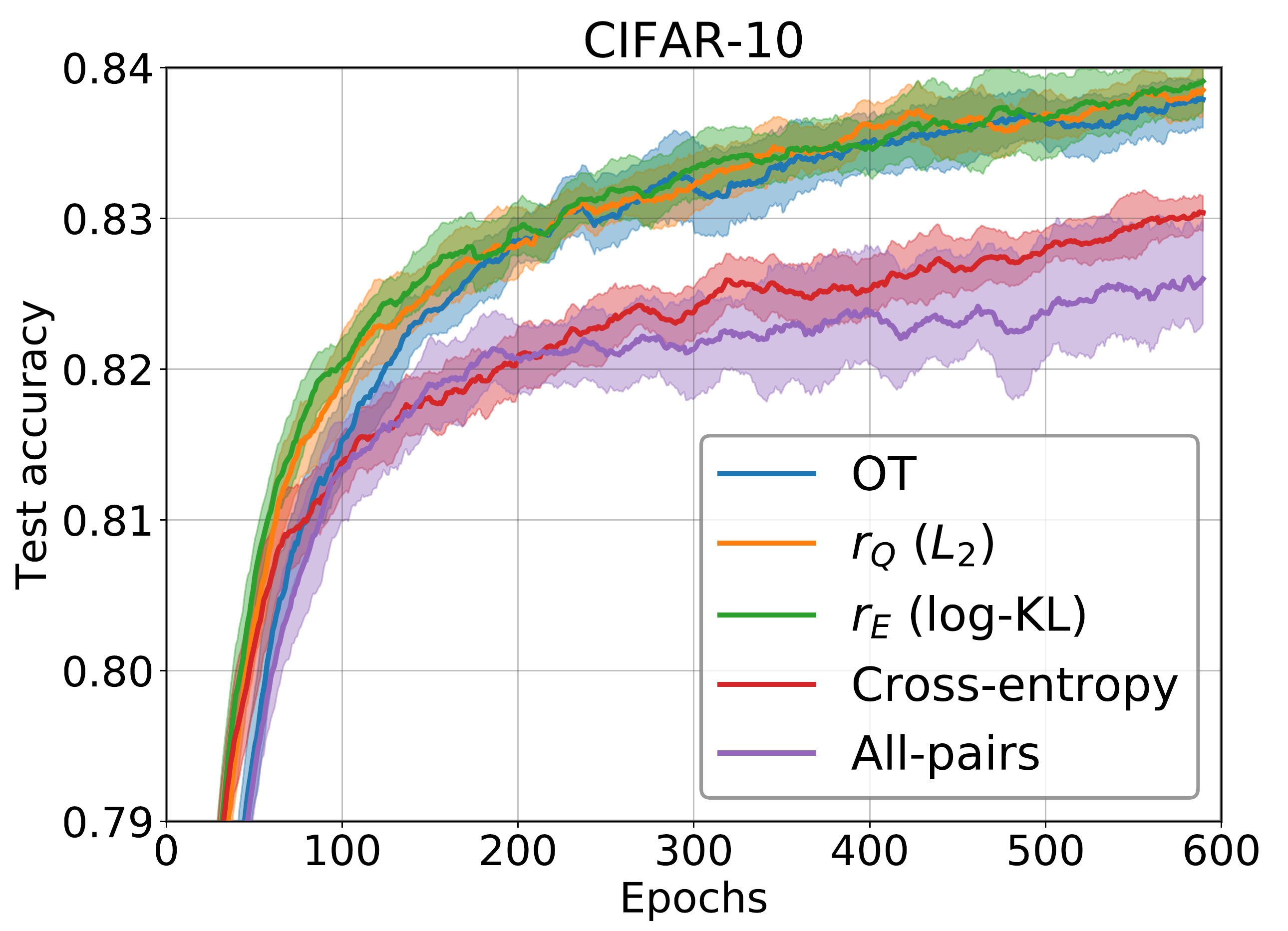} 
    \hspace{0.2cm}
    \includegraphics[height=3.8cm]{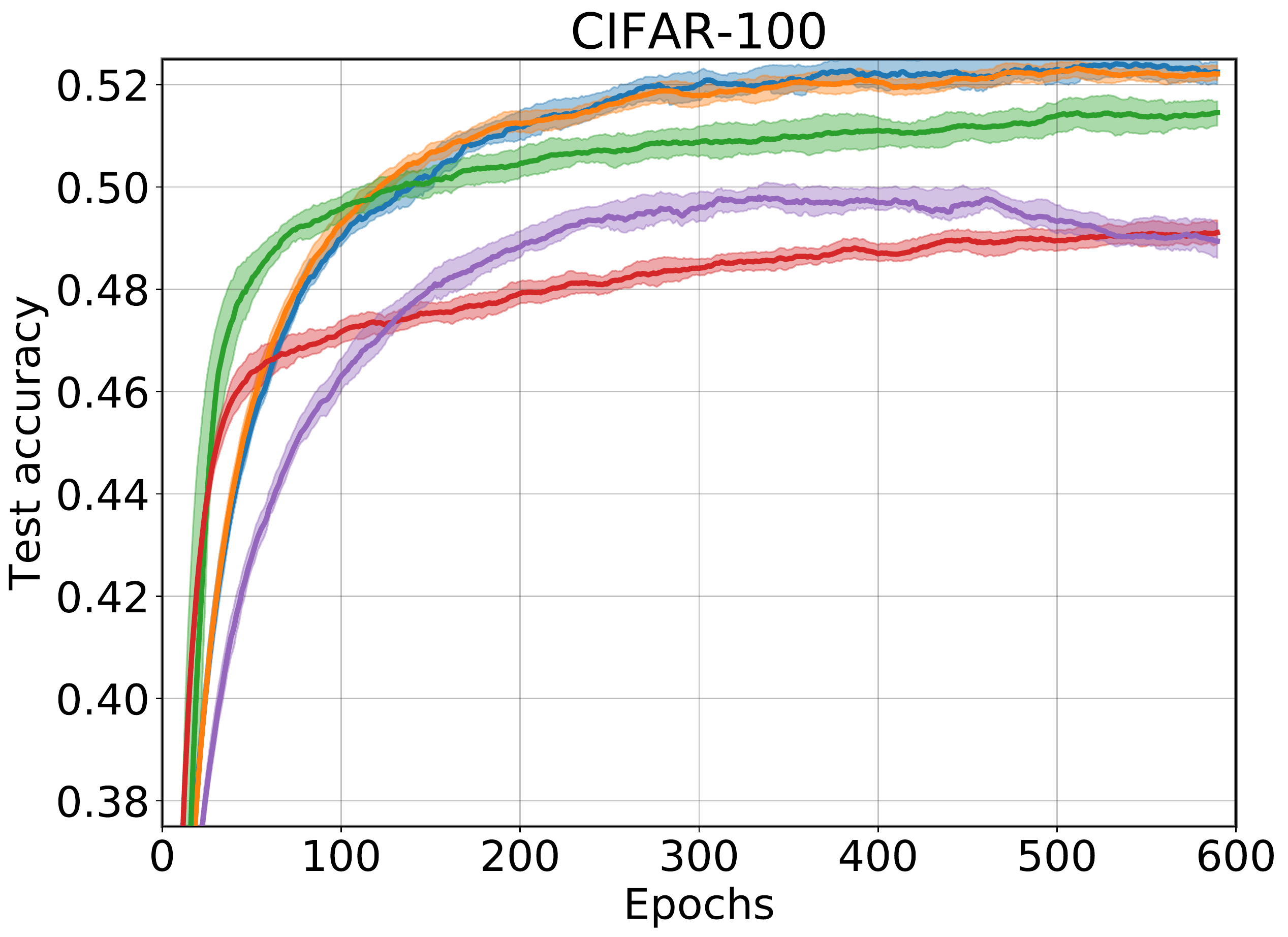} 
    \hspace{0.2cm}
    \includegraphics[height=3.8cm]{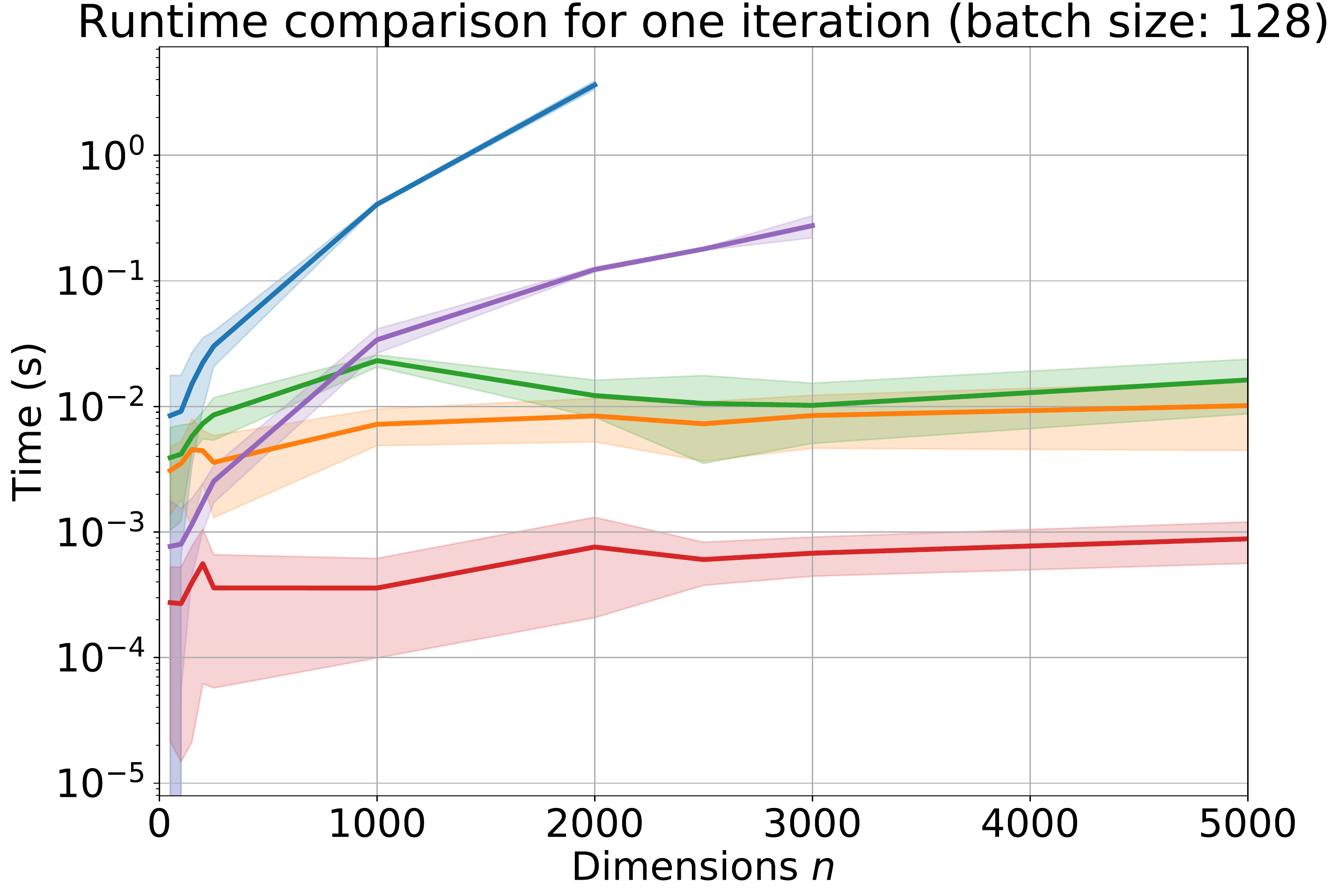} 
    \caption{{\bf Left, center:} Accuracy comparison on CIFAR-10, CIFAR-100
        ($n=10$, $n=100$).
        {\bf Right:} Runtime comparison for one batch computation with
        backpropagation disabled. OT and
        All-pairs go out-of-memory starting from $n=2000$ and $n=3000$,
        respectively. With backpropagation enabled, the runtimes are similar but
        OT and All-pairs go out-of-memory at $n=1000$ and $n=2500$,
        respectively.
}
    \label{fig:cifar}
\end{figure*}

In the quadratic regularization case, it was already derived by
\citet{djolonga_2017} that $\mathbf{B}^Q_j \coloneqq \ones / |\cB_j|$.
The multiplication with the Jacobian,
$\bnu \coloneqq
\frac{\partial \v_Q(\s, \w)}{\partial \s}
\u$ for some vector $\u$, can be computed as
$\bnu = (\bnu_1, \dots, \bnu_m)$, 
where $\bnu_j = \textnormal{mean}(\u_{\cB_j}) \ones \in \RR^{|\cB_j|}$.
In the entropic regularization case, novel to our knowledge,
we have $\mathbf{B}^E_j = \ones \otimes \textnormal{softmax}(\s_{\cB_j})$. 
Note that $\mathbf{B}^E_j$ is column-wise constant, so that
the multiplication with the Jacobian
$\bnu \coloneqq 
\frac{\partial \v_E(\s, \w)}{\partial \s} \u$,
can be computed as
$\bnu_j = \langle \textnormal{softmax}(\s_{\cB_j}), 
\u_{\cB_j}\rangle \ones \in \RR^{|\cB_j|}$.
In both cases, the multiplication with the Jacobian therefore takes $O(n)$ time.

There are interesting differences between the two forms
of regularization.
For quadratic regularization, the Jacobian only depends on the partition
$\cB_1,\dots,\cB_m$ (not on $\s$) and the blocks have constant value.
For entropic regularization, the Jacobian does depend on $\s$ and the blocks are
constant column by column.
Both formulations are averaging the incoming gradients, one uniformly
and the other weighted.

\paragraph{Differentiating the projections.}

We now combine \Cref{prop:projection} with \Cref{lemma:Jacobian_isotonic} to
characterize the Jacobians of the projections onto the permutahedron and show
how to multiply arbitrary vectors with them in linear time.
\begin{proposition}{Jacobian of the projections}
\label{prop:Jacobian_projection}
 
Let $P_\Psi(\z, \w)$ be defined in Proposition \ref{prop:projection}.
Then, 
\begin{equation}
\frac{\partial P_{\Psi}(\z, \w)}{\partial \z}
= \J_\Psi(\z_{\sigma(\z)}, \w)_{\sigmainv(\z)},
\end{equation}
where $\J_\pi$ is the matrix obtained by permuting the rows \underline{and}
columns of $\J$ according to $\pi$, and where
\begin{equation}
\J_\Psi(\s, \w) \coloneqq
\I - \frac{\partial \v_\Psi(\s, \w)}{\partial \s}.
\end{equation}
\end{proposition}
Again, the Jacobian w.r.t.\ $\w$ is entirely symmetric.
Unlike the Jacobian of isotonic optimization, the Jacobian of the
projection is not block diagonal, as we need to permute its rows and columns.
We can nonetheless multiply with it in linear time
by using the simple identity $(\J_\pi) \z = (\J \z_{\pi^{-1}})_\pi$, 
which allows us to reuse the $O(n)$ multiplication with the Jacobian of
isotonic optimization.

\paragraph{Differentiating $s_{\varepsilon \Psi}$ and $r_{\varepsilon \Psi}$.}

With the Jacobian of $P_\Psi(\z, \w)$ w.r.t. $\z$ and $\w$ at hand, 
differentiating $s_{\varepsilon \Psi}$ and $r_{\varepsilon \Psi}$
boils down to a mere application of the chain rule to \eqref{eq:soft_sort}
and \eqref{eq:soft_rank}. To summarize, we can multiply with the Jacobians
of our soft operators in $O(n)$ time and space.

\section{Experiments}
\label{sec:experiments}

We present in this section our empirical findings. 
NumPy, JAX, PyTorch and Tensorflow versions of our sorting and ranking operators
are available at
\url{https://github.com/google-research/fast-soft-sort/}.

\subsection{Top-k classification loss function}

\paragraph{Experimental setup.}

To demonstrate the effectiveness of our proposed soft rank operators as a
drop-in replacement for exisiting ones, we reproduce the top-$k$
classification experiment of \citet{cuturi_2019}. The authors propose a loss for
top-$k$ classification between a ground truth class $y \in [n]$ and a vector of
soft ranks $\r \in \RR^n$, which is higher if the predicted soft ranks correctly
place $y$ in the top-$k$ elements. We compare the following soft operators
\begin{itemize}[topsep=0pt,itemsep=2pt,parsep=2pt,leftmargin=10pt]
    \item OT \citep{cuturi_2019}: optimal transport formulation.
    \item All-pairs \citep{qin_2010}: noting that $[r(\btheta)]_i$ is
        equivalent to $\sum_{j \neq i} \mathbf{1}[\theta_i < \theta_j] + 1$, one can
        obtain soft ranks in $O(n^2)$ by replacing the indicator function with a
        sigmoid.
    \item Proposed: our $O(n \log n)$ soft ranks $r_Q$ and $r_E$.
        Although not used in this experiment, for top-$k$ ranking, the
        complexity can be reduced to $O(n \log k)$ by computing $P_\Psi$ using
        the algorithm of \citet{projection_permutahedron}.
\end{itemize}

We use the CIFAR-10 and CIFAR-100 datasets, with $n=10$ and $n=100$ classes,
respectively.
Following \citet{cuturi_2019}, we use a vanilla CNN (4 Conv2D with 2
max-pooling layers, ReLU activation, 2 fully connected layers with batch norm on
each),
the ADAM optimizer \citep{kingma_2014} with a constant step size of
$10^{-4}$, and set $k=1$. Similarly to \citet{cuturi_2019}, we found that
squashing the scores $\btheta$ to $[0,1]^n$ with a logistic map was beneficial.

\vspace{-0.3cm}
\paragraph{Results.}

Our empirical results, averaged over $12$ runs, are shown in Figure
\ref{fig:cifar} (left, center).  On both CIFAR-10 and CIFAR-100, our soft rank
formulations achieve comparable accuracy to the OT formulation, though
significantly faster, as we elaborate below. Similarly to 
\citet{cuturi_2019}, we found that the soft top-$k$ loss slightly
outperforms the classical cross-entropy (logistic) loss for these two datasets.
However, we did not find that the All-pairs formulation could outperform the
cross-entropy loss. 

The training times for 600 epochs on CIFAR-100 were
29 hours (OT), 21 hours ($r_Q$), 23 hours ($r_E$) and
16 hours (All-pairs). Training times on CIFAR-10 were similar.
While our soft operators are several hours faster than OT, 
they are slower than All-pairs, despite its $O(n^2)$ complexity. 
This is due the fact that, with $n=100$, All-pairs is very efficient on GPUs,
while our PAV implementation runs on CPU. 

\subsection{Runtime comparison: effect of input dimension}

To measure the impact of the dimensionality $n$ on the runtime of each
method, we designed the following experiment.

\vspace{-0.3cm}
\paragraph{Experimental setup.}

We generate score vectors $\btheta \in \RR^n$ randomly according to
$\mathcal{N}(0, 1)$,  for $n$ ranging from $100$ up to
$5000$.  For fair comparison with GPU implementations (OT, All-pairs,
Cross-entropy), we create a batch of $128$ such vectors and we compare the time
to compute soft ranking operators on this batch. We run this experiment on top
of TensorFlow \citep{tensorflow} on a six core Intel Xeon W-2135 with 64 GBs of
RAM and a GeForce GTX 1080 Ti.

\vspace{-0.3cm}
\paragraph{Results.}

Run times for one batch computation with backpropagation disabled are shown in
Figure \ref{fig:cifar} (Right). While their runtime is reasonable in small
dimension, OT and All-pairs scale quadratically with respect to the
dimensionality $n$ (note the log scale on the $y$-axis).
Although slower than a softmax, our formulations scale well,
with the dimensionality $n$ having negligible impact on the runtime.  OT and
All-pairs go out-of-memory starting from $n=2000$ and $n=3000$, respectively.
With backpropagation enabled, they go out-of-memory at $n=1000$ and $n=2500$,
due to the need for recording the computational graph. This shows that the lack
of memory available on GPUs is problematic for these methods.  In contrast, our
approaches only require $O(n)$ memory and comes with the theoretical Jacobian
(they do not rely on differentiating through iterates). They therefore suffer
from no such issues.

\subsection{Label ranking via soft Spearman's rank correlation coefficient}

We now consider the label ranking setting where supervision is given as full
rankings (e.g., $2 \succ 1 \succ 3 \succ 4$) rather than as label relevance
scores. The goal is therefore to learn to predict permutations, i.e.,
a function $f_\w \colon \cX \to \Sigma$. A classical metric between ranks is 
Spearman's rank correlation coefficient, defined as the Pearson correlation
coefficient between the ranks. Maximizing this coefficient
is equivalent to minimizing the squared loss between ranks.
A naive idea would be therefore to use as loss
$\frac{1}{2} \|\r - r(\btheta)\|^2$, where $\btheta = g_\w(\x)$.
This is unfortunately a discontinuous function of $\btheta$.
We therefore propose to rather use
$\frac{1}{2} \|\r - r_\Psi(\btheta)\|^2$, hence the name differentiable
Spearman's rank correlation coefficient. At test time, we replace $r_\Psi$ with
$r$, which is justified by the order-preservation property (Proposition
\ref{prop:properties}).

\vspace{-0.3cm}
\paragraph{Experimental setup.}

We consider the $21$ datasets from \citep{hullermeier2008label,cheng09icml},
which has both semi-synthetic data obtained from classification problems, and
real biological measurements.
Following \citep{korba2018structured}, we average over two 10-fold validation
runs, in each of which we train on 90\% and evaluate on 10\% of the data.
Within each repetition, we run an internal 5-fold cross-validation to
grid-search for the best parameters.
We consider linear models of the form $g_{W,\b}(\x) = W \x + \b$,
and for ablation study we drop the soft ranking layer $r_\Psi$.

\vspace{-0.3cm}
\paragraph{Results.}

Due to the large number of datasets, we choose to present a summary of the
results in \Cref{fig:label_ranking_spearman}. 
We postpone detailed results to the appendix (\Cref{table:label_ranking_spearman}).
Out of 21 datasets, introducing a soft rank layer with $\Psi=Q$ works better on 15
datasets, similarly on 4 and worse on 2 datasets.
We can thus conclude that even for such simple model, introducing our layer
is beneficial, and even achieving state of the art results on some of the
datasets (full details in the appendix).

\begin{figure}[t]
    \centering
    \includegraphics[width=0.31\textwidth]{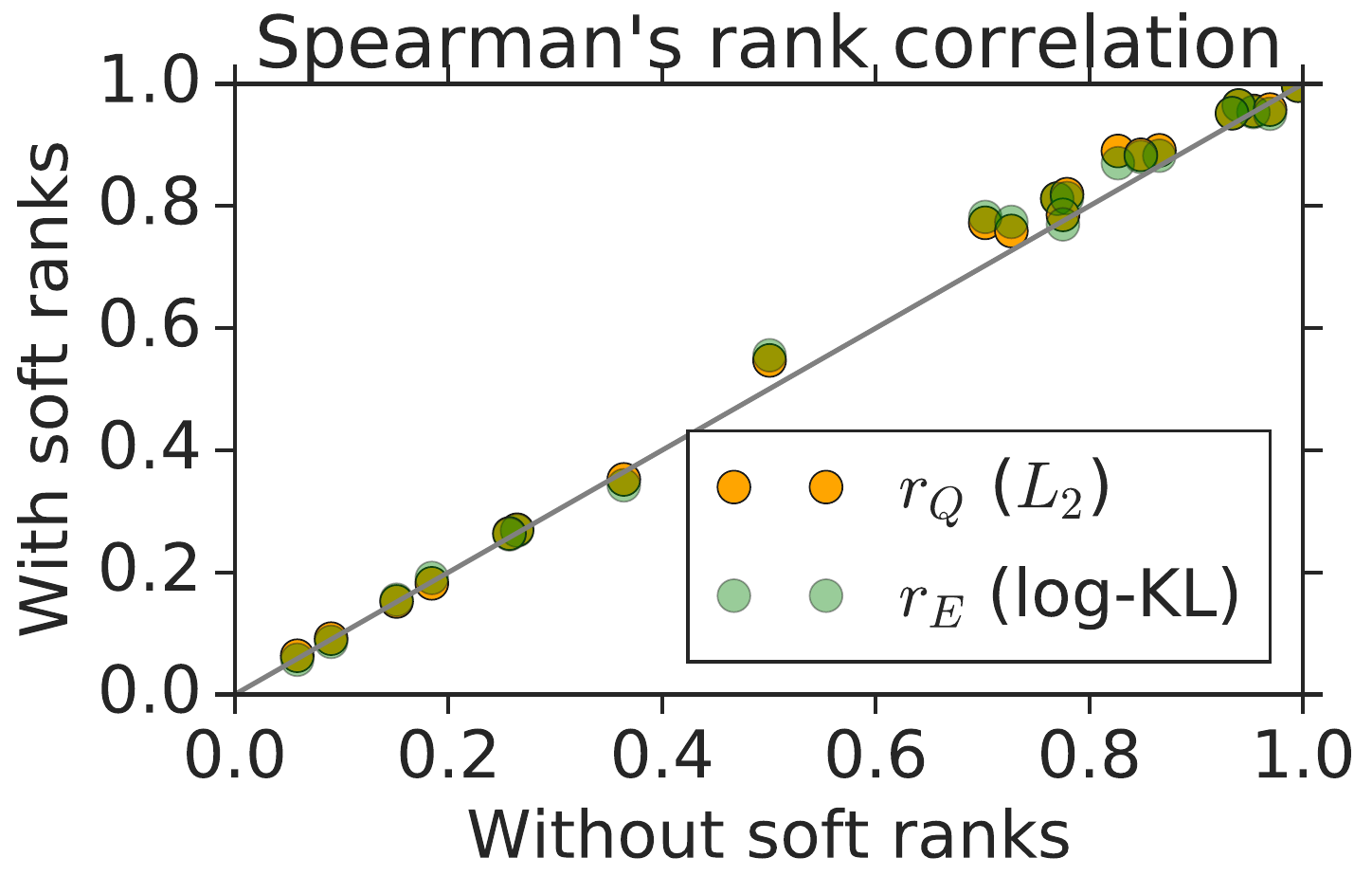} 
    \caption{{\bf Label ranking accuracy} with and without soft rank layer. Each
        point above the line represents a dataset where our soft rank layer
    improves Spearman's rank correlation coefficient.}
\label{fig:label_ranking_spearman}
\end{figure}

\begin{figure*}[t]
    \centering
	\begin{minipage}{.28\textwidth}
    \vspace{.28cm}
    \includegraphics[width=0.96\textwidth]{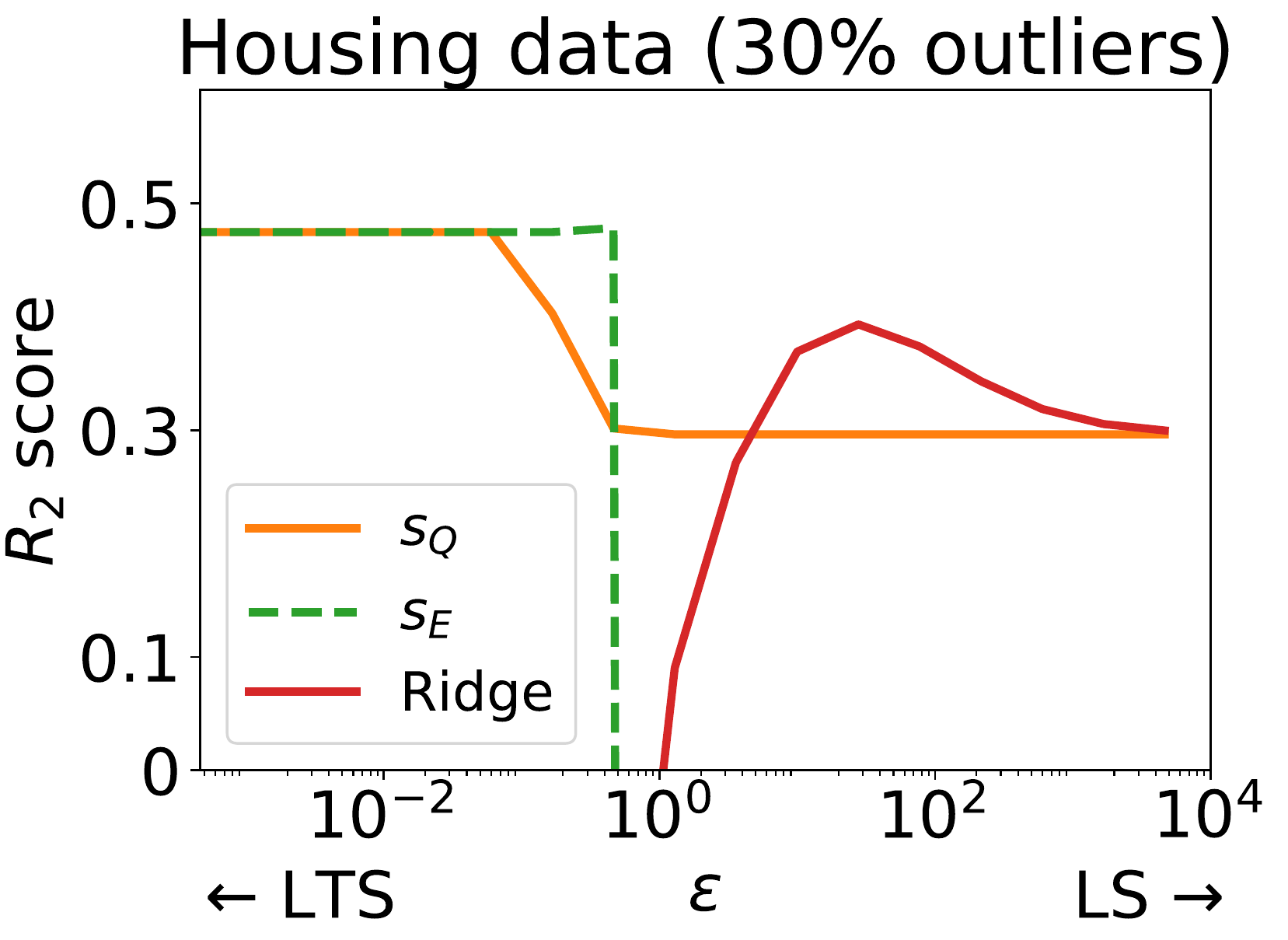} 
		\captionof{figure}{Empirical validation of interpolation between LTS and LS.}
		\label{fig:ls-interpolation}
	\end{minipage}
	\hspace*{\fill}
	\begin{minipage}{.7\textwidth}
        \includegraphics[width=0.96\textwidth]{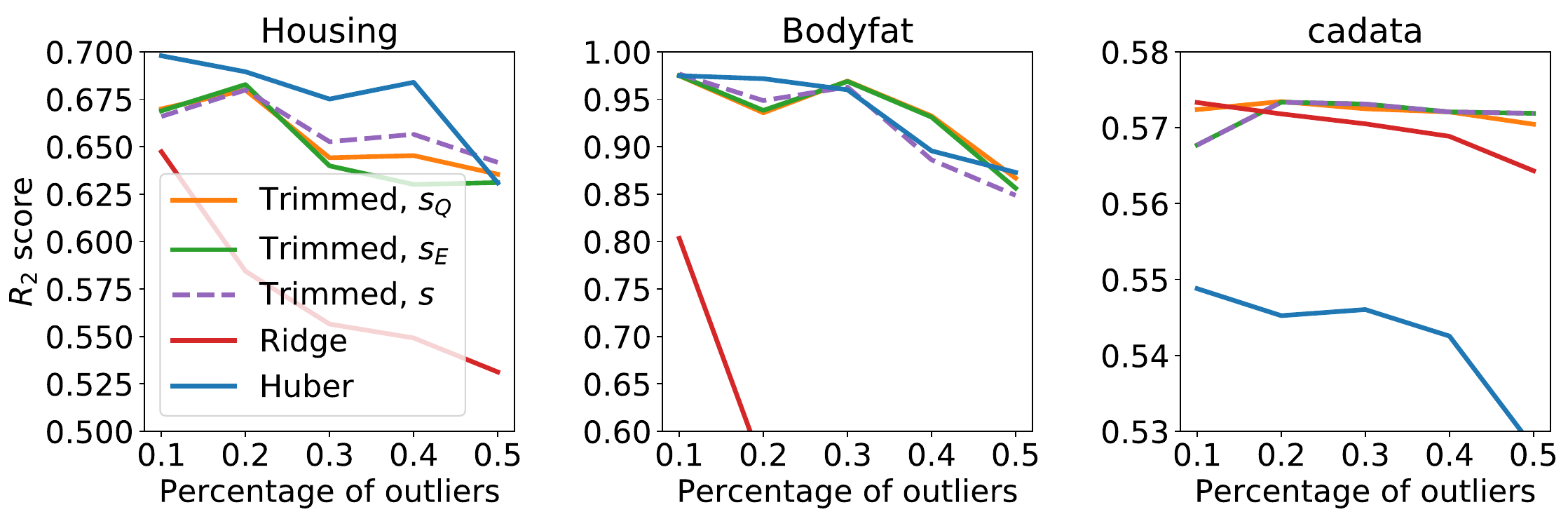} 
        \captionof{figure}{$R_2$ score (higher is better) averaged over
    $10$ train-time splits for increasing percentage of outliers.
Hyper-parameters are tuned by $5$-fold cross-validation.}
		\label{fig:trimmed_lineplot}
	\end{minipage}
\end{figure*}

\subsection{Robust regression via soft least trimmed squares}
\label{sec:trimmed_regression}

We explore in this section the application of our soft sorting operator
$s_{\varepsilon \Psi}$ to robust regression. 
Let $\x_1, \dots, \x_n \in \cX \subseteq \RR^d$ and $y_1,\dots,y_n \in \cY
\subseteq \RR$ be a training set of input-output pairs.
Our goal is to learn a model $g_\w \colon \RR^d \to \RR$ that predicts outputs
from inputs, where $\w$ are model parameters. 
We focus on $g_\w(\x) \coloneqq \langle \w, \x \rangle$ for simplicity.
We further assume that a
certain proportion of examples are outliers including some label noise, which
makes the task of robustly estimating $g_\w$ particularly challenging.  

The classical ridge regression can be cast as
\begin{equation}
\min_\w \frac{1}{n} \sum_{i=1}^n \ell_i(\w) + \frac{1}{2\varepsilon} \|\w\|^2,
\label{eq:ridge_regression}
\end{equation}
where $\ell_i(\w) \coloneqq \frac{1}{2} (y_i - g_\w(\x_i))^2$. 
In order to be robust to label noise,
we propose instead to sort the losses (from larger to smaller) and to ignore the
first $k$ ones. Introducing our soft sorting operator, this can be formulated as
\begin{equation}
\min_\w \frac{1}{n - k} \sum_{i=k+1}^n \ell_i^\varepsilon(\w)
\label{eq:soft_trimmed}
\end{equation}
where $\ell_i^\varepsilon(\w) \coloneqq [s_{\varepsilon \Psi}(\bm{\ell}(\w))]_i$
is the $i^{\text{th}}$ loss in the soft sort,
and  $\bm{\ell}(\w) \in \RR^n$ is the loss vector that gathers
$\ell_i(\w)$ for each $i \in [n]$. 
Solving \eqref{eq:soft_trimmed} with existing $O(n^2)$ soft sorting operators
could be particularly computationally prohibitive, since here $n$ is the number
of training samples.

When $\varepsilon \to 0$, 
we have
$s_{\varepsilon \Psi}(\bm{\ell}(\w)) \to s(\bm{\ell}(\w))$
and \eqref{eq:soft_trimmed} is known as least trimmed squares
(LTS) \citep{rousseeuw_1984,rousseeuw_2005}.
When $\varepsilon \to \infty$,
we have, from Proposition \ref{prop:properties},
$s_{\varepsilon \Psi}(\bm{\ell}(\w)) \to \text{mean}(\bm{\ell}(\w)) \ones$
and therefore both \eqref{eq:ridge_regression} and \eqref{eq:soft_trimmed}
converge to the least squares (LS) objective,
$\min_\w \frac{1}{n} \sum_{i=1}^n \ell_i(\w)$.
To summarize, our proposed objective \eqref{eq:soft_trimmed}, dubbed
soft least trimmed squares, \textbf{interpolates} between least trimmed squares
($\epsilon \to 0$) and least squares ($\epsilon \to \infty$), 
as also confirmed empirically in \Cref{fig:ls-interpolation}.

\vspace{-0.2cm}
\paragraph{Experimental setup.}

To empirically validate our proposal, we compare cross-validated results for
increasing percentage of outliers of the following methods:
\begin{itemize}[topsep=0pt,itemsep=2pt,parsep=2pt,leftmargin=10pt]
    \item Least trimmed squares, with truncation parameter $k$,
    \item Soft least trimmed squares \eqref{eq:soft_trimmed}, 
        with truncation parameter $k$ and regularization parameter
        $\varepsilon$,
    \item Ridge regression \eqref{eq:ridge_regression}, with regularization
        parameter $\varepsilon$,
    \item Huber loss \citep{huber_1964} with regularization parameter
        $\varepsilon$ and threshold parameter $\tau$, 
        as implemented in scikit-learn \citep{sklearn}.
\end{itemize}
We consider datasets from the LIBSVM archive
\citep{libsvm_datasets}.  We hold out 20\% of the data as test set and use the
rest as training set. We artifically create outliers, by adding noise to a
certain percentage of the training labels, using
$y_i \leftarrow y_i + e$, where $e \sim \mathcal{N}(0, 5\times \text{std}(\y))$.  
We do not add noise to the test set. 
For all methods, we use L-BFGS \citep{lbfgs}, with a maximum of $300$
iterations.
For hyper-parameter optimization, we use $5$-fold cross-validation.  We choose
$k$ from $\{\lceil 0.1 n\rceil, \lceil 0.2 n\rceil, \dots, \lceil 0.5
n\rceil\}$, $\varepsilon$ from $10$ log-spaced values between $10^{-3}$ and
$10^4$, and $\tau$ from $5$ linearly spaced values between $1.3$ and $2$. We
repeat this procedure $10$ times with a different train-test split, and report
the averaged $R_2$ scores (a.k.a.\ coefficient of determination).

\vspace{-0.2cm}
\paragraph{Results.}

The averaged $R_2$ scores (higher is better)
are shown in Figure \ref{fig:trimmed_lineplot}.
On all datasets, the accuracy of ridge regression deteriorated
significantly with increasing number of outliers. Least trimmed squares (hard or
soft) performed slightly worse than the Huber loss on housing, 
comparably on bodyfat and much better on cadata. We found that hard least trimmed
squares (i.e., $\varepsilon = 0$) worked well on all datasets, showing that
regularization is less important for sorting operators (which are piecewise
linear) than for ranking operators (which are piecewise constant).
Nevertheless, regularization appeared useful in some cases.
For instance, on cadata, the cross-validation procedure
picked $\varepsilon > 1000$ when the percentage of outliers is less 
than $20 \%$, and $\varepsilon < 10^{-3}$ when the percentage of outliers is
larger than $20\%$. 
This is confirmed visually
on Figure \ref{fig:trimmed_lineplot}, where the soft sort with $\Psi=Q$ works
slightly better than the hard sort with few outliers, then performs comparably
with more outliers.
The interpolation effect enabled by $\varepsilon$ therefore allows some 
adaptivity to the (unknown) percentage of outliers. 

\section{Conclusion}

Building upon projections onto permutahedra, we constructed differentiable
sorting and ranking operators.  We derived \textbf{exact} $O(n \log n)$
computation and $O(n)$ differentiation of these operators, a key technical
contribution of this paper.  We demonstrated that our operators can be used
as a drop-in replacement for existing $O(n^2)$ ones, with an order-of-magnitude
speed-up. We also showcased two applications enabled by our soft operators:
label ranking with differentiable Spearman's rank correlation coefficient and
robust regression via soft least trimmed squares.

\clearpage

\section*{Acknowledgements}

We are grateful to Marco Cuturi and Jean-Philippe Vert for useful discussions,
and to Carlos Riquelme for comments on a draft of this paper.



\appendix
\onecolumn

\begin{center}
\Huge{Appendix} 
\end{center}

\section{Additional empirical results}

We include in this section the detailed label ranking results on the same $21$
datasets as considered by \citet{hullermeier2008label} as well
as \citet{cheng09icml}.

For entropic regularization $E$, in addition to $r_E$, we also consider an
alternative formulation.
Since $\bre$ is already strictly positive, instead of using the log-projection
onto $\cP(e^{\bre})$, we can directly use the projection onto $\cP(\bre)$. In
our notation, this can be written as
$\tilde{r}_{\varepsilon E}(\btheta) = \tilde{r}_E(\nicefrac{\btheta}{\varepsilon})$, where
\begin{equation}
    \tilde{r}_E(\btheta)
\coloneqq \argmin_{\bmu \in \cP(\bre)}
\textnormal{KL}(\bmu, e^{-\btheta})
= e^{P_E(-\btheta, \log \bre)}.
\end{equation}
Spearman's rank correlation coefficient for each method, averaged over $5$ runs,
is shown in the table below.

\begin{table}[H]
\centering
\begin{tabular}{lllll}
\toprule
{} &                          $r_Q$ ($L_2$) &                         $r_E$ (log-KL) &                     $\tilde{r}_E$ (KL) &                          No projection \\
Dataset                &                                        &                                        &                                        &                                        \\
\midrule
fried                  &  1.00 $\pm$ 0.00                       &  1.00 $\pm$ 0.00                       &  1.00 $\pm$ 0.00                       &  1.00 $\pm$ 0.00                       \\
wine                   &  0.96 $\pm$ 0.03 ({\color{red} -0.01}) &  0.95 $\pm$ 0.04 ({\color{red} -0.02}) &  0.96 $\pm$ 0.03 ({\color{red} -0.01}) &  0.97 $\pm$ 0.02                       \\
authorship &  0.96 $\pm$ 0.01 &  0.95 $\pm$ 0.01                       &  0.95 $\pm$ 0.01                       &  0.95 $\pm$ 0.01                       \\
pendigits              &  0.96 $\pm$ 0.00 ({\color{teal}+0.02}) &  0.96 $\pm$ 0.00 ({\color{teal}+0.02}) &  0.96 $\pm$ 0.00 ({\color{teal}+0.02}) &  0.94 $\pm$ 0.00                       \\
segment                &  0.95 $\pm$ 0.01 ({\color{teal}+0.02}) &  0.95 $\pm$ 0.00 ({\color{teal}+0.02}) &  0.95 $\pm$ 0.01 ({\color{teal}+0.02}) &  0.93 $\pm$ 0.01                       \\
glass                  &  0.89 $\pm$ 0.04 ({\color{teal}+0.03}) &  0.88 $\pm$ 0.05 ({\color{teal}+0.02}) &  0.89 $\pm$ 0.04 ({\color{teal}+0.03}) &  0.87 $\pm$ 0.05                       \\
vehicle                &  0.88 $\pm$ 0.02 ({\color{teal}+0.04}) &  0.88 $\pm$ 0.02 ({\color{teal}+0.03}) &  0.89 $\pm$ 0.02 ({\color{teal}+0.04}) &  0.85 $\pm$ 0.03                       \\
iris                   &  0.89 $\pm$ 0.07 ({\color{teal}+0.06}) &  0.87 $\pm$ 0.07 ({\color{teal}+0.04}) &  0.87 $\pm$ 0.07 ({\color{teal}+0.05}) &  0.83 $\pm$ 0.09                       \\
stock                  &  0.82 $\pm$ 0.02 ({\color{teal}+0.04}) &  0.81 $\pm$ 0.02 ({\color{teal}+0.03}) &  0.83 $\pm$ 0.02 ({\color{teal}+0.05}) &  0.78 $\pm$ 0.02                       \\
wisconsin              &  0.79 $\pm$ 0.03 ({\color{teal}+0.01}) &  0.77 $\pm$ 0.03 ({\color{red} -0.01}) &  0.79 $\pm$ 0.03 ({\color{teal}+0.01}) &  0.78 $\pm$ 0.03                       \\
elevators              &  0.81 $\pm$ 0.00 ({\color{teal}+0.04}) &  0.81 $\pm$ 0.00 ({\color{teal}+0.04}) &  0.81 $\pm$ 0.00 ({\color{teal}+0.04}) &  0.77 $\pm$ 0.00                       \\
vowel                  &  0.76 $\pm$ 0.03 ({\color{teal}+0.03}) &  0.77 $\pm$ 0.01 ({\color{teal}+0.05}) &  0.78 $\pm$ 0.02 ({\color{teal}+0.05}) &  0.73 $\pm$ 0.02                       \\
housing                &  0.77 $\pm$ 0.03 ({\color{teal}+0.07}) &  0.78 $\pm$ 0.02 ({\color{teal}+0.08}) &  0.77 $\pm$ 0.03 ({\color{teal}+0.07}) &  0.70 $\pm$ 0.03                       \\
cpu-small              &  0.55 $\pm$ 0.01 ({\color{teal}+0.05}) &  0.56 $\pm$ 0.01 ({\color{teal}+0.05}) &  0.54 $\pm$ 0.01 ({\color{teal}+0.04}) &  0.50 $\pm$ 0.02                       \\
bodyfat                &  0.35 $\pm$ 0.07 ({\color{red} -0.01}) &  0.34 $\pm$ 0.07 ({\color{red} -0.02}) &  0.34 $\pm$ 0.08 ({\color{red} -0.02}) &  0.36 $\pm$ 0.07                       \\
calhousing             &  0.27 $\pm$ 0.01 ({\color{teal}+0.01}) &  0.27 $\pm$ 0.01 &  0.27 $\pm$ 0.01 ({\color{teal}+0.01}) &  0.26 $\pm$ 0.01                       \\
diau                   &  0.26 $\pm$ 0.02                       &  0.26 $\pm$ 0.02                       &  0.26 $\pm$ 0.02                       &  0.26 $\pm$ 0.02                       \\
spo                    &  0.18 $\pm$ 0.02                       &  0.19 $\pm$ 0.02 ({\color{teal}+0.01}) &  0.18 $\pm$ 0.02                       &  0.18 $\pm$ 0.02                       \\
dtt                    &  0.15 $\pm$ 0.04                       &  0.16 $\pm$ 0.04 &  0.14 $\pm$ 0.04 ({\color{red} -0.01}) &  0.15 $\pm$ 0.04                       \\
cold                   &  0.09 $\pm$ 0.03                       &  0.09 $\pm$ 0.03                       &  0.10 $\pm$ 0.03 ({\color{teal}+0.01}) &  0.09 $\pm$ 0.04                       \\
heat                   &  0.06 $\pm$ 0.02                       &  0.06 $\pm$ 0.02                       &  0.06 $\pm$ 0.02                       &  0.06 $\pm$ 0.02                       \\
\bottomrule
\end{tabular}
\caption{Detailed results of our label ranking experiment. Blue color indicates
better Spearman rank correlation coefficient compared to using no projection. 
Red color indicates worse coeffcient.}
\label{table:label_ranking_spearman}
\end{table}

\clearpage
\section{Proofs}

\subsection{Proof of Lemma \ref{lemma:discrete} (Discrete optimization
formulation)}
\label{appendix:proof_lemma_discrete}

For the first claim, we have for all $\w \in \RR^n$ such that 
$w_1 > w_2 > \dots > w_n$
\begin{equation}
\sigma(\btheta) = \argmax_{\sigma \in \Sigma} \langle \btheta_\sigma, \w \rangle
\label{eq:sigma_discrete_all_w}
\end{equation}
and in particular for $\w = \bre$.
The second claim follows from
\begin{equation}
\sigma(\btheta) 
= \argmax_{\sigma \in \Sigma} \langle \btheta, \w_\sigmainv \rangle
= \argmax_{\pi^{-1} \in \Sigma} \langle \btheta, \w_\pi \rangle
= \left(\argmax_{\pi \in \Sigma} \langle \btheta, \w_\pi \rangle\right)^{-1}.
\end{equation}

\subsection{Proof of Proposition \ref{prop:lp} (Linear programming formulations)}
\label{appendix:proof_prop_lp}

Let us prove the first claim.
The key idea is to absorb $\btheta_\sigma$ in the permutahedron.
Using \eqref{eq:sigma_discrete_all_w}, we obtain for all $\btheta \in \RR^n$ and
for all $\w \in \RR^n$ such that $w_1 > \dots > w_n$
\begin{equation}
\btheta_{\sigma(\btheta)} 
= \argmax_{\btheta_\sigma \colon \sigma \in \Sigma} \langle \btheta_\sigma, \w
\rangle
= \argmax_{\y \in \Sigma(\btheta)} \langle \y, \w \rangle
= \argmax_{\y \in \cP(\btheta)} \langle \y, \w \rangle,
\end{equation}
where in the second equality we used $\cP(\btheta) = \conv(\Sigma(\btheta))$ and
the fundamental theorem of linear programming \citep[Theorem 6]{dantzig}.
For the second claim, we have similarly
\begin{equation}
\w_{r(\btheta)} 
= \argmax_{\w_\pi \colon \pi \in \Sigma} \langle \btheta, \w_\pi \rangle
= \argmax_{\y \in \cP(\w)} \langle \btheta, \y \rangle.
\end{equation}
Setting $\w = \bre$ and using 
$\bre_{r(\btheta)} = \bre_{\sigmainv(\btheta)} = \sigmainv(-\btheta) =
r(-\btheta)$ proves the claim.

\subsection{Proof of Proposition \ref{prop:properties} (Properties of soft
sorting and ranking operators)}
\label{appendix:prop_soft_sort_rank}

\paragraph{Differentiability.}

Let $\cC$ be a closed convex set and let 
$\bmu^\star(\z) \coloneqq \argmax_{\bmu \in \cC} \langle \bmu, \z \rangle -
\Psi(\z)$.  If $\Psi$ is strongly convex over $\cC$, then $\bmu^\star(\z)$ is
Lipschitz continuous. By Rademacher’s theorem, $\bmu^\star(\z)$ is
differentiable almost everywhere.
Furthermore, since $P_\Psi(\z, \w) = \nabla \Psi(\bmu^\star(\z))$ with $\cC
=\cP(\nabla \Psi^{-1}(\w))$,
$P_\Psi(\z, \w)$ is differentiable a.e.\ as long as $\Psi$ is
twice differentiable, which is the case when $\Psi \in \{Q,E\}$.

\paragraph{Order preservation.}

Proposition 1 of \citet{fy_losses_journal} shows that
$\bmu^\star(\z)$ and $\z$ are sorted the same way. 
Furthermore, since $P_\Psi(\z, \w) = \nabla \Psi(\bmu^\star(\z))$ with $\cC
=\cP(\nabla \Psi^{-1}(\w))$ and since $\nabla \Psi$ is monotone,
$P_\Psi(\z, \w)$ is sorted the same way as $\z$, as well.
Let $\s = s_{\varepsilon \Psi}(\btheta)$ and
$\r = r_{\varepsilon \Psi}(\btheta)$.
From the respective definitions, this means that $\s$ is sorted the same way as
$\bre$ (i.e., it is sorted in descending order) and $\r$ is sorted the same way
as $-\btheta$, which concludes the proof.

\paragraph{Asymptotic behavior.}

We will now characterize the behavior for sufficiently small and large
regularization strength $\varepsilon$. Note that rather than multiplying the
regularizer $\Psi$ by $\varepsilon > 0$, we instead divide $\s$ by
$\varepsilon$, which is equivalent.
\begin{lemma}{Analytical solutions of isotonic optimization in the limit regimes}

If $\varepsilon \le \epsmin(\s, \w)\coloneqq \min_{i\in[n-1]}
\frac{s_i-s_{i+1}}{w_i-w_{i+1}}$, then
\begin{equation}
\v_Q(\s/\varepsilon, \w)=\v_E(\s/\varepsilon, \w)=\s/\varepsilon - \w.
\end{equation}
If $\varepsilon > \epsmax(\s, \w)\coloneqq \max_{i<j}
\frac{s_i-s_{j}}{w_i-w_{j}}$, then
\begin{equation}
v_Q(\s/\varepsilon, \w)= \frac{1}{n} \sum_{i=1}^n(s_i/\varepsilon - w_i) \ones
\quad \text{and} \quad
v_E(\s/\varepsilon, \w)=(\text{LSE}(\s / \varepsilon) - \text{LSE}(\w))\ones,
\end{equation}
where $\text{LSE}(\x) \coloneqq \log \sum_i e^{x_i}$.
\end{lemma}
\begin{proof}
We start with the $\varepsilon \le \epsmin(\s, \w)$ case.
Recall that $\s$ is sorted in descending order.
Therefore, since we chose $\varepsilon$ sufficiently small, 
the vector $\v=\s/\varepsilon-\w$ is sorted in descending order
as well. This means that $\v$ is feasible, i.e., it
belongs to the constraint sets in \Cref{prop:projection}.
Further, note that $v_i=\gamma_Q(\{i\}; \s / \varepsilon, \w)=\gamma_E(\{i\}; \s
/ \varepsilon, \w) = s_i / \varepsilon - w_i$ so that $\v$ is the optimal
solution if we drop the constraints, which completes the argument.

Next, we tackle the $\varepsilon > \epsmax(\s, \w)$ case.
Note that the claimed solutions are exactly $\gamma_Q([n]; \s, \w)$
and $\gamma_E([n]; \s, \w)$, so the claim will immediately follow if
we show that $[n]$ is an optimal partition.
The PAV algorithm (cf.\ \S\ref{appendix:pav}) merges at each iteration any two
neighboring blocks $B_1,B_2$ that violate $\gamma_\Psi(B_1 ; \s / \varepsilon,
\w) \geq \gamma_\Psi(B_2; \s / \varepsilon, \w)$, starting
from the partitions consisting of singleton sets.
Let $k \in \{1, \dots, n-1\}$ be the iteration number.
We claim that the two blocks, $B_1=\{1, 2, \ldots, k\}$ and $B_2=\{k+1\}$,
will always be violating the constraint, so that they can be merged.
Note that in the quadratic case, they can be merged only if
\begin{equation}
\sum_{i=1}^k (s_i/\varepsilon - w_i)/k <  s_{k+1}/\varepsilon - w_{k+1},
\end{equation}
which is equivalent to
\begin{equation}
\sum_{i=1}^k \frac{s_i - s_{k+1}}{k\varepsilon} < \sum_{i=1}^k(w_i-w_{k+1}),
\end{equation}
which is indeed satisfied when 
$\varepsilon > \epsmax(\s, \w)$.
In the KL case, they can be merged only if
\begin{align}
	\log \sum_{i=1}^k e^{s_i/\varepsilon} - \log \sum_{i=1}^k e^{w_i} <
	s_{k+1}/\varepsilon - w_{k+1}
	&\iff \log \sum_{i=1}^k e^{s_i/\varepsilon} - s_{k+1}/\varepsilon <
	\log \sum_{i=1}^k e^{w_i} - w_{k+1} \\
	&\iff \log \sum_{i=1}^k e^{s_i/\varepsilon} - \log e^{s_{k+1}/\varepsilon} <
	\log \sum_{i=1}^k e^{w_i} - \log e^{w_{k+1}} \\
	&\iff \log \sum_{i=1}^k e^{(s_i-s_{k+1})/\varepsilon} <
	\log \sum_{i=1}^k e^{w_i - w_{k+1}} \\
	&\iff \sum_{i=1}^k e^{(s_i-s_{k+1})/\varepsilon} <
	\sum_{i=1}^k e^{w_i-w_{k+1}}.
\end{align}
This will be true if the $i^\text{th}$ term on the left-hand side is smaller
than the $i^\text{th}$ term on the right-hand side, i.e., when
$(s_i-s_{k+1})/\varepsilon < w_i-w_{k+1}$, which again is implied by the
assumption.
\end{proof}

We can now directly characterize the behavior of the projection operator
$P_\Psi$ in the two regimes $\varepsilon \le \epsmin(s(\z), \w)$ and $\varepsilon
> \epsmax(s(\z), \w)$. This in turn implies the results for both the soft ranking and
sorting operations using \eqref{eq:soft_sort} and \eqref{eq:soft_rank}.
\begin{proposition}{Analytical solutions of the projections in the limit regimes}

If $\varepsilon \le \epsmin(s(\z), \w)$, then
\begin{equation}
P_\Psi(\z/\varepsilon, \w)=\w_{\sigma^{-1}(\z)}.
\end{equation}
If $\varepsilon > \epsmax(s(\z), \w)$, then
\begin{align}
P_Q(\z/\varepsilon, \w) 
&= \z / \varepsilon - \text{mean}(\z / \varepsilon - \w)\ones, 
\textrm{ and}\\
P_E(\z/\varepsilon, \w) 
&= \z/\varepsilon - \text{LSE}(\z / \varepsilon) \ones + 
\text{LSE}(\w) \ones.
\end{align}
\end{proposition}
Therefore, in these two regimes, we do not even need PAV to compute the optimal
projection.

\clearpage
\subsection{Proof of Proposition \ref{prop:projection} (Reduction to isotonic
optimization)}
\label{appendix:proof_forward}

Before proving Proposition \ref{prop:projection}, we need
the following three lemmas.

\begin{lemma}{Technical lemma}
\label{lemma:technical_lemma}

Let $f \colon \RR \to \RR$ be convex, $v_1 \ge v_2$ and $s_2 \ge s_1$. Then,
$f(s_1 - v_1) + f(s_2 - v_2) \ge f(s_2 - v_1) + f(s_1 - v_2)$.
\end{lemma}
\begin{proof}
Note that 
$s_2 - v_2 \ge s_2 - v_1 \ge s_1 - v_1$
and
$s_2 - v_2 \ge s_1 - v_2 \ge s_1 - v_1$.
This means that we can express
$s_2 - v_1$ and $s_1 - v_2$ as a convex combination
of the endpoints of the line segment $[s_1 - v_1, s_2 - v_2]$,
namely 
\begin{equation}
s_2 - v_1 = \alpha (s_2 - v_2) + (1 - \alpha) (s_1 - v_1) 
\quad \text{and} \quad
s_1 - v_2 = \beta (s_2 - v_2) + (1 - \beta) (s_1 - v_1).
\end{equation}
Solving for $\alpha$ and $\beta$ gives $\alpha = 1-\beta$.
From the convexity of $f$, we therefore have
\begin{equation}
f(s_2 - v_1) \le \alpha f(s_2 - v_2) + (1 - \alpha) f(s_1 - v_1)
\quad \text{and} \quad
f(s_1 - v_2) \le (1-\alpha) f(s_2 - v_2) + \alpha f(s_1 - v_1).
\end{equation}
Summing the two proves the claim.
\end{proof}

\begin{lemma}{Dual formulation of a regularized linear program}
\label{lemma:dual_regularized_LP}

Let $\bmu^\star = \argmax_{\bmu \in \cC} \langle \bmu, \z \rangle -
\Psi(\bmu)$, where $\cC \subseteq \RR^n$ is a closed convex set and $\Psi$ is
strongly convex. Then, the
corresponding dual solution is $\u^\star = \argmin_{\u \in \RR^n}
\Psi^*(\z - \u) + s_\cC(\u)$, where $s_\cC(\u) \coloneqq \sup_{\y \in \cC}
\langle \y, \u \rangle$ is the support function of $\cC$.  Moreover, $\bmu^\star
= \nabla \Psi^*(\z - \u^\star)$.
\end{lemma}

\begin{proof}
The result is well-known and we include the proof for
completeness. Let us define the Fenchel conjugate of a function $\Omega
\colon \RR^n \to \RR \cup \{\infty\}$ by
\begin{equation}
\Omega^*(\z) \coloneqq \sup_{\bmu \in \RR^n} \langle \bmu, \z \rangle
- \Omega(\bmu).
\end{equation}
Let $\Omega \coloneqq \Psi + \Phi$, where $\Psi$ is strongly convex and $\Phi$
is convex.  We have
\begin{equation}
\Omega^*(\z) 
= (\Psi + \Phi)^*(\z) 
= \inf_{\u \in \RR^n} \Phi^*(\u) + \Psi^*(\z - \u),
\end{equation}
which is the infimal convolution of $\Phi^*$ with $\Psi^*$.
Moreover, $\nabla \Omega^*(\z) = \nabla \Psi^*(\z - \u^\star)$.
The results follows from choosing $\Phi(\bmu) = I_\cC(\bmu)$ and noting that
$I_\cC^* = s_\cC$.
\end{proof}
For instance, with $\Psi = Q$, we have $\Psi^* = Q$, and with $\Psi = E$,
we have $\Psi^* = \exp$.

The next lemma shows how to go further by choosing $\cC$ as the base polytope
$\cB(F)$ associated with a cardinality-based submodular function $F$, of which
the permutahedron is a special case. The polytope is defined
as (see, e.g., \citet{bach_2013})
\begin{equation}
    \cB(F) \coloneqq \left\{ \bmu \in \RR^n \colon \sum_{i \in \cS} \mu_i \le
    F(\cS) ~ \forall \cS \subseteq [n], \sum_{i=1}^n \mu_i = F([n]) \right\}.
\end{equation}
\begin{lemma}{Reducing dual formulation to isotonic regression}

Let $F(\cS) = g(|\cS|)$ for some concave $g$. Let
$\cB(F)$ be its corresponding base polytope.  Let $\sigma$ be a permutation of
$[n]$ such that $\z \in \RR^n$ is sorted in descending order, i.e.,
$z_{\sigma_1} \ge z_{\sigma_2} \ge \dots \ge z_{\sigma_n}$. 
Assume $\Psi(\bmu) = \sum_{i=1}^n \psi(\mu_i)$, where $\psi$ is convex.
Then, the dual solution $\u^\star$ from Lemma \ref{lemma:dual_regularized_LP} is
equal to $\v^\star_\sigmainv$, where
\begin{align}
\v^\star
&= \argmin_{v_1 \ge \dots \ge v_n} 
\Psi^*(\z_\sigma - \v) +  \langle \f_\sigma, \v \rangle \\
&= -\argmin_{v'_1 \le \dots \le v'_n} 
\Psi^*(\v'_\sigma + \z) -  \langle \f_\sigma, \v' \rangle.
\end{align}
\label{lemma:dual_isotonic}
\end{lemma}
\begin{proof}
The support function $s_{\cB(F)}(\u)$ is known
as the Lov\'{a}sz extension of $F$. For conciseness, we use the standard
notation $f(\u) \coloneqq s_{\cB(F)}(\u)$. Applying Lemma
\ref{lemma:dual_regularized_LP}, we obtain
\begin{equation}
\u^\star = \argmin_{\u \in \RR^n} \Psi^*(\z - \u) + f(\u).
\end{equation}
Using the ``greedy algorithm'' of Edmonds (\citeyear{edmonds_1970}), we can
compute $f(\u)$ as follows. First, choose a permutation $\sigma$ that sorts $\u$
in descending order, i.e., $u_{\sigma_1} \ge u_{\sigma_2} \ge \dots \ge
u_{\sigma_n}$.  Then a maximizer $\f\in \cB(F) \subseteq \RR^n$ is obtained by
forming $\f_\sigma = (f_{\sigma_1},\dots,f_{\sigma_n})$, where
\begin{equation}
f_{\sigma_i} \coloneqq 
F(\{\sigma_1, \dots, \sigma_i\}) - F(\{\sigma_1, \dots, \sigma_{i-1}\}).
\end{equation}
Moreover, $\langle \f, \u \rangle = f(\u)$.

Let us fix $\sigma$ to the permutation that sorts $\u^\star$.
Following the same idea as from \cite{djolonga_2017},
since the Lov\'{a}sz extension is linear on the set of all vectors that are
sorted by $\sigma$, we can write
\begin{align}
\argmin_{\u \in \RR^n} \Psi^*(\z - \u) + f(\u)
&= \argmin_{u_{\sigma_1} \ge \dots \ge u_{\sigma_n}} 
\Psi^*(\z - \u) + \langle \f, \u \rangle.
\end{align}
This is an instance of isotonic optimization, as we can rewrite the problem as
\begin{align}
\argmin_{v_1 \ge \dots \ge v_n} 
\Psi^*(\z - \v_\sigmainv) +  \langle \f, \v_{\sigmainv} \rangle
&= \argmin_{v_1 \ge \dots \ge v_n} 
\Psi^*(\z_\sigma - \v) +  \langle \f_\sigma, \v \rangle,
\label{eq:reduction_isotonic_regression}
\end{align}
with $\u_\sigma^\star = \v^\star \Leftrightarrow \u^\star =
\v_{\sigma^{-1}}^\star$.

Let $\s \coloneqq \z_\sigma$. It remains to show that $s_1 \ge \dots \ge s_n$,
i.e., that $\s$ and the optimal dual variables $\v^\star$ are both in descending
order. 
Suppose $s_j > s_i$ for some $i < j$.
Let $\s'$ be a copy of $\s$ with $s_i$ and $s_j$ swapped.
Since $\psi^*$ is convex,
by Lemma \ref{lemma:technical_lemma}, 
\begin{equation}
\Psi^*(\s - \v^\star) - \Psi^*(\s' - \v^\star)
= \psi^*(s_i - v_i^\star) + \psi^*(s_j - v_j^\star)
- \psi^*(s_j - v_i^\star) - \psi^*(s_i - v_j^\star)
\ge 0,
\end{equation}
which contradicts the assumption that $\v^\star$ and the corresponding $\sigma$
are optimal. A similar result is proven by \citet[Lemma 1]{suehiro_2012} but for
the optimal primal variable $\bmu^\star$.
\end{proof}
We now prove Proposition \ref{prop:projection}.
The permutahedron $\cP(\w)$ is a special case of $\cB(F)$ with
$F(\cS) = \sum_{i=1}^{|\cS|} w_i$ and $w_1 \ge w_2 \ge \dots \ge w_n$.
In that case, $\f_\sigma = (f_{\sigma_1}, \dots, f_{\sigma_n}) = (w_1, \dots,
w_n) = \w$. 

For $\cP(\nabla \Psi^*(\w))$, we thus have $\f_\sigma = \nabla \Psi^*(\w)$.
Finally, note that if $\Psi$ is Legendre-type, which is the case of
both $Q$ and $E$, then $\nabla \Psi^* = (\nabla \Psi)^{-1}$.
Therefore, $\nabla \Psi(\bmu^\star) = \z - \u^\star$, which concludes the
proof.

\subsection{Relaxed dual linear program interpretation}
\label{appendix:relaxed_dual_lp}

We show in this section that the dual problem in Lemma \ref{lemma:dual_isotonic}
can be interpreted as the original dual linear program (LP) with relaxed
equality constraints. Consider the primal LP
\begin{equation}
    \max_{\y \in \cB(F)} \langle \y, \z \rangle.
    \label{eq:primal_lp}
\end{equation}
As shown by \citet[Proposition 3.2]{bach_2013}, the dual LP is
\begin{equation}
    \min_{\blambda \in \cC} \sum_{\cS \subseteq \cV} \lambda_\cS F(\cS)
    \label{eq:dual_lp}
\end{equation}
where
\begin{equation}
\cC \coloneqq \left\{ \blambda \in \RR^{2^\cV} \colon 
\lambda_\cS \ge 0 ~ \forall \cS \subset \cV, 
\lambda_\cV \in \RR, 
z_i = \sum_{\cS \colon i \in \cS} \lambda_\cS ~ \forall i \in [n] 
\right\}.
\end{equation}
Moreover, let $\sigma$ be a permutation sorting $\z$ in descending order.
Then, an optimal $\blambda$ is given by \citep[Proposition 3.2]{bach_2013}
\begin{equation}
\lambda_\cS = \left\{
\begin{array}{rl}
    z_{\sigma_i} - z_{\sigma_{i+1}} & \text{if }  \cS =
    \{\sigma_1,\dots,\sigma_i\} \\
    z_{\sigma_n} & \text{if }  \cS = \{\sigma_1,\dots,\sigma_n\} \\
    0 & \text{otherwise. }
\end{array}\right.
\end{equation}
Now let us restrict to the support of $\blambda$ and do the change of variable
\begin{equation}
\lambda_\cS = \left\{
\begin{array}{rl}
    v_i - v_{i+1} & \text{if }  \cS = \{\sigma_1,\dots,\sigma_i\} \\
    v_n & \text{if }  \cS = \{\sigma_1,\dots,\sigma_n\}.
\end{array}\right.
\end{equation}
The non-negativity constraints in $\cC$ become
$v_1 \ge v_2 \ge \dots \ge v_n$
and the equality constraints in $\cC$ become $\z_\sigma = \v$.
Adding quadratic regularization $\frac{1}{2} \|\y\|^2$ in the primal
problem \eqref{eq:primal_lp} is equivalent to relaxing the dual equality
constraints in \eqref{eq:dual_lp} by smooth constraints $\frac{1}{2}
\|\z_\sigma - \v\|^2$ (this can be seen by adding quadratic regularization
to the primal variables of \citet[Eq. (3.6)]{bach_2013}).
For the dual objective \eqref{eq:dual_lp}, we have
\begin{align}
\sum_{\cS \subseteq \cV} \lambda_\cS F(\cS) 
&= \sum_{i=1}^{n-1} (v_i - v_{i+1}) F(\{\sigma_1, \dots, \sigma_i\})
+ v_n F(\{\sigma_1,\dots,\sigma_n\}) \\
&= \sum_{i=1}^n (F(\{\sigma_1,\dots,\sigma_i\}) -
F(\{\sigma_1,\dots,\sigma_{i-1}\})) v_i \\
&= \langle \f_\sigma, \v \rangle,
\end{align}
where in the second line we used \citep[Eq.\ (3.2)]{bach_2013}.
Altogether, we obtain $\min_{v_1 \ge \dots \ge v_n} \frac{1}{2} \|\z_\sigma
- \v\|^2 +  \langle \f_\sigma, \v \rangle$, which is exactly the expression we
derived in Lemma \ref{lemma:dual_isotonic}.
The entropic case is similar.

\subsection{Pool adjacent violators (PAV) algorithm}
\label{appendix:pav}

Let $g_1,\dots,g_n$ be convex functions.
As shown in \cite{best_2000_separable,projection_permutahedron},
\begin{equation}
    \argmin_{v_1 \ge \dots \ge v_n} \sum_{i=1}^n g_i(v_i)
\end{equation}
can be solved using a generalization of the PAV algorithm (note that unlike
these works, we use decreasing constraints for convenience).
All we need is a routine for solving, given some set $\cB$ of indices, the
``pooling'' sub-problem
\begin{equation}
    \argmin_{\gamma \in \RR} \sum_{i \in \cB} g_i(\gamma).
\end{equation}
Thus, we can use PAV to solve \eqref{eq:reduction_isotonic_regression}, as long
as $\Psi^*$ is separable. We now give the closed-form solution for two special
cases. To simplify, we denote
$\s \coloneqq \z_\sigma$ and $\w \coloneqq \f_\sigma$.

\paragraph{Quadratic regularization.}

We have $g_i(v_i) = \frac{1}{2} (s_i - v_i)^2 + v_i w_i$.
We therefore minimize
\begin{equation}
\sum_{i \in \cB} g_i(\gamma) = \sum_{i \in \cB} \frac{1}{2} (s_i - \gamma)^2 
+ \gamma \sum_{i \in \cB} w_i.
\end{equation}
The closed-form solution is
\begin{equation}
    \gamma^\star_Q(\s, \w; \cB) = \frac{1}{|\cB|} \sum_{i \in \cB} (s_i - w_i).
\end{equation}

\paragraph{Entropic regularization.}

We have $g_i(v_i) = e^{s_i - v_i} + v_i e^{w_i}$.
We therefore minimize
\begin{equation}
\sum_{i \in \cB} g_i(\gamma) 
= \sum_{i \in \cB} e^{s_i - \gamma} + \gamma \sum_{i \in \cB} e^{w_i}.
\end{equation}
The closed-form solution is
\begin{equation}
\gamma^\star_E(\s, \w; \cB) 
= - \log \frac{\sum_{i \in \cB} w_i}{\sum_{i \in \cB} e^{s_i}}
= \text{LSE}(\s_\cB) - \text{LSE}(\w_\cB),
\end{equation}
where $\text{LSE}(\x) \coloneqq \log \sum_i e^{x_i}$.

Although not explored in this work, other regularizations are potentially
possible, see, e.g., \citep{fy_losses}.

\subsection{Proof of Proposition \ref{prop:Jacobian_projection} (Jacobian of
isotonic optimization)}
\label{appendix:Jacobian_isotonic}

Let $\cB_1,\dots,\cB_m$ be the partition of $[n]$ induced by $\v \coloneqq
\v_\Psi(\s, \w)$. From the PAV algorithm, for all $i \in [n]$, there is a
unique block $\cB_l \in \{\cB_1,\dots,\cB_m\}$ such that $i \in \cB_l$ and $v_i
= \gamma_\Psi(\cB_l; \s, \w)$. Therefore, for all $i \in [n]$, we obtain
\begin{equation}
\partialfrac{v_i}{s_j} 
= \left\{
   \begin{array}{cl}
\partialfrac{\gamma_\Psi(\cB_l; \s, \w)}{s_j} 
& \text{if }  i,j \in \cB_l \\
0 &\text{otherwise. }
   \end{array}\right.
\end{equation}
Therefore, the Jacobian matrix is block diagonal, i.e.,
\begin{equation}
    \frac{\partial \v}{\partial \s} = 
\begin{bmatrix} 
    \mathbf{B}^\Psi_1 & \zeros & \zeros \\
    \zeros            & \ddots & \zeros \\
    \zeros            & \zeros & \mathbf{B}^\Psi_m
\end{bmatrix}.
\end{equation}
For the block $\cB_l$, the non-zero partial derivatives form a matrix
$\B^\Psi_l \in \RR^{|\cB_l| \times |\cB_l|}$ such that each column is
associated with one $s_j$ and contains
the value $\partialfrac{\gamma_\Psi(\cB_l; \s, \w)}{s_j}$ (all values in a
column are the same). 
For quadratic regularization, we have
\begin{equation}
\partialfrac{v_i}{s_j} =
\left\{
   \begin{array}{cl}
\frac{1}{|\cB_l|}
& \text{if }  i,j \in \cB_l \\
0 &\text{otherwise. }
   \end{array}\right.
\end{equation}
For entropic regularization, we have
\begin{equation}
\partialfrac{v_i}{s_j} =
\left\{
   \begin{array}{cl}
       \frac{e^{s_j}}{\sum_{j' \in \cB} e^{s_{j'}}} =
       \text{softmax}(\s_{\cB_l})_j
& \text{if }  i,j \in \cB_l \\
0 &\text{otherwise. }
   \end{array}\right.
\end{equation}
The multiplication with the Jacobian uses the fact that each block is constant
column-wise.

\paragraph{Remark.}

The expression above is for points $\s$ where $\v$ is differentiable.
For points where $\v$ is not differentiable, we can take an arbitrary matrix in
the set of Clarke's generalized Jacobians, the convex hull of Jacobians of the
form $\lim_{\s_t \to \s} \partial \v / \partial \s_t$.
The points of non-differentiability occur when a block of the optimal solution
can be split up into two blocks with equal values.  In that case, the two
directional derivatives do not agree, but are derived for quadratic
regularization by \citet{djolonga_2017}.

\end{document}